\documentclass[twoside,11pt]{article}

\usepackage[preprint,nohyperref]{arxiv}

\usepackage{graphicx}
\usepackage{wrapfig}
\usepackage{subcaption}
\usepackage{float}
\usepackage{placeins}
\usepackage{enumitem}
\usepackage[dvipsnames,table]{xcolor}
\usepackage{amssymb}
\usepackage{amsmath}
\usepackage{amsfonts}
\usepackage{mathtools}
\usepackage{bm}
\usepackage{bbm}
\usepackage{multirow}
\usepackage{booktabs}
\usepackage{longtable}
\usepackage{microtype}
\DisableLigatures[f]{encoding = *, family = *} 
\usepackage{nicefrac}
\usepackage[utf8]{inputenc}
\usepackage[T1]{fontenc}
\usepackage{natbib}
\usepackage{tikz}
\usepackage{minitoc} 

\usepackage[ruled,linesnumbered]{algorithm2e}
\usepackage{algorithmic}
\usepackage{varwidth}
\SetAlgoNlRelativeSize{0} 
\SetCommentSty{} 

\usepackage{hyperref}
\usepackage{url}
\usepackage{xurl}

\makeatletter
\def\theHALG@line{\arabic{algocf}.\arabic{ALG@line}}
\makeatother

\usepackage[capitalize,noabbrev]{cleveref}
\crefname{section}{Section}{Sections}
\crefname{subsection}{Section}{Sections}
\crefname{appendix}{Appendix}{Appendices}
\crefname{figure}{Figure}{Figures}
\crefname{table}{Table}{Tables}
\crefname{algorithm}{Algorithm}{Algorithms}
\crefname{theorem}{Theorem}{Theorems}

\crefformat{equation}{Eq.~(#2#1#3)}
\crefrangeformat{equation}{Eqs.~(#3#1#4)~to~(#5#2#6)}
\crefmultiformat{equation}{Eqs.~(#2#1#3)}{,~(#2#1#3)}{, (#2#1#3)}{,~(#2#1#3)}
\crefrangemultiformat{equation}{Eqs.~(#3#1#4)~to~(#5#2#6)}{,~(#3#1#4)~to~(#5#2#6)}{, (#3#1#4)~to~(#5#2#6)}{,~(#3#1#4)~to~(#5#2#6)}

\begin{document}

\title{\large WUSH-KV: KV Cache Quantization with Data-Adaptive Transforms}

\author{\name Jiale Chen
\email Jiale.Chen@ist.ac.at
\hfill
\addr Institute of Science and Technology Austria (ISTA)
\AND
\name Vage Egiazarian
\hfill
\addr Institute of Science and Technology Austria (ISTA)
\AND
\name Eldar Kurtić
\hfill
\addr Institute of Science and Technology Austria (ISTA) \& Red Hat AI
\AND
\name Torsten Hoefler
\hfill
\addr ETH Zürich
\AND
\name Dan Alistarh\thanks{This author was affiliated with Red Hat AI during part of this work.}
\hfill
\addr Institute of Science and Technology Austria (ISTA)
}

\maketitle

\begin{abstract}
KV cache memory and bandwidth costs grow with context length and batch size, which limits efficient long-context inference.
To address this bottleneck, we introduce WUSH-KV for low-bit KV-cache quantization.
It adapts WUSH, which constructs a data-aware transform from the second-order statistics of both factors in a matrix product to reduce quantization error.
WUSH-KV uses calibration data to construct separate key and value transforms, with the value transform folded into the model weights and the key transform applied after RoPE.
The transforms can be paired with clipped quantizers.
For one such quantizer, QuEST INT, we show that, under mild assumptions, the WUSH transform is near-optimal.
With this quantizer, WUSH-KV reduces layerwise reconstruction error and achieves the lowest end-to-end perplexity among other tested transforms.
For end-to-end evaluation, we integrate WUSH-KV into SGLang using OSCAR-style percentile-clipped affine quantization.
At 2-bit, WUSH-KV performs comparably to or outperforms the OSCAR transform across all evaluated models and downstream tasks.
\end{abstract}

\section{Introduction}
\label{sec:intro}

Large language models are increasingly used with long contexts and large serving batches.
This puts growing pressure on the memory capacity and bandwidth of inference systems.
One important source of this cost is the key-value (KV) cache.
During autoregressive generation, each attention layer stores the keys and values of previous tokens so that they do not need to be recomputed.
The cache grows linearly with sequence length and batch size, and it is repeatedly read as new tokens are generated.
As a result, the KV cache can become a major bottleneck for long-context inference.
Low-bit quantization provides a simple way to reduce its memory footprint and data movement.
Yet, aggressive quantization can introduce large errors into both attention scores and outputs.

Changes of basis (linear transforms) can make cached vectors easier to quantize, but the quality of the transform should also reflect how quantization errors propagate through attention.
Key errors affect query-key products, whereas value errors affect projected attention outputs.
This motivates separate transforms that account for both cache statistics and downstream sensitivity.
WUSH~\citep{chen2026wush} was recently introduced as a principled approach to joint weight-activation quantization.
It constructs a closed-form invertible transform from the second-order statistics of both factors in a matrix product.
This product-aware view naturally extends to the KV cache.

In this work, we adapt WUSH from weight-activation quantization to KV-cache quantization.
We call the resulting method \emph{WUSH-KV}.
For each KV head, WUSH-KV constructs separate transforms for keys and values using calibration-time statistics.
The key transform accounts for the query heads that consume the cached keys.
The value transform accounts for the output-projection blocks that consume the cached values.
The value-side transform can be folded into the model weights and therefore adds no online transform cost.
The key-side transform and its query-side compensation are applied after headwise normalization (when present) and rotary positional encoding.
The WUSH transforms operate independently of the scalar quantizer and can be paired with per-token clipped quantizers.
WUSH-KV also retains sink and recent cache entries in full precision.
Our theoretical analysis focuses on the QuEST~\citep{pmlr-v267-panferov25a} quantizer.
Under an additive rounding-noise model and explicit conditions on normalized tails and clipping errors, we show that the ideal WUSH transform is near-optimal among transforms with balanced coordinate sensitivity.
The controlled reconstruction study uses the same quantizer to isolate transform quality, while the end-to-end evaluations use the percentile-clipped affine quantizer from the OSCAR~\citep{zhou2026oscarofflinespectralcovarianceaware} concurrent work.
In our experiments, WUSH-KV consistently reduces attention reconstruction error and achieves the lowest perplexity among all existing methods.
In 2-bit downstream benchmarks, it scores higher than OSCAR on all four Qwen3-8B tasks and is competitive at 4B and 32B.

\section{Related Work}
\label{sec:related_work}

\textbf{KV cache quantization.}
There is a long line of work on this topic. 
KIVI~\citep{pmlr-v235-liu24bz} observes that keys and values have different quantization characteristics, so it quantizes keys per channel and values per token. Recent residual keys and values remain in full precision.
KVQuant~\citep{NEURIPS2024_028fcbcf} also exploits the channel-wise structure of keys.
It quantizes keys before RoPE, uses sensitivity-weighted nonuniform quantization, and handles outliers with a per-vector dense-and-sparse representation.
It keeps the first token in FP16 to protect the attention sink.
AQUA-KV~\citep{pmlr-v267-shutova25a} exploits cross-layer dependencies with compact adapters to predict cached keys and values and quantizes the remaining residual information.
These methods focus on quantizer design, outlier handling, and selective high-precision retention.

\textbf{Transform-based quantization.}
Another line of work changes the representation before quantization.
QuaRot~\citep{NEURIPS2024_b5b93943} redistributes outliers with randomized Hadamard transforms.
This enables end-to-end 4-bit quantization of weights, activations, and the KV cache.
SpinQuant~\citep{ICLR2025_e5b1c0d4} shows that quantized accuracy can vary substantially across rotations.
It learns two mergeable orthogonal rotations by minimizing a task loss on calibration data.
For low-bit activations and KV caches, it also uses efficient fixed online Hadamard transforms.
FlatQuant~\citep{pmlr-v267-sun25l} learns Kronecker-structured transforms through calibration and applies head-wise transforms to keys and values in the KV cache.
RotateKV~\citep{ijcai2025p690} specializes in rotations for aggressive 2-bit cache compression.
It combines calibration-based channel reordering with grouped-head key rotations applied before RoPE.
It also identifies additional attention-sink tokens and retains their KV entries in FP16.
TurboQuant~\citep{ICLR2026_5c802ef3} randomly rotates vectors and applies scalar quantizers matched to the resulting coordinate distribution.
For unbiased inner-product estimation, the published method additionally applies a one-bit-per-coordinate QJL sketch to the quantization residual.
Subsequent analysis~\citep{benbasat2026noteturboquantearlierdriveeden} notes that randomized Hadamard transforms can replace uniform random rotations in this family of quantizers.
Practical implementations (e.g., vLLM\footnote{\url{https://docs.vllm.ai/en/v0.29.0/api/vllm/model_executor/layers/quantization/turboquant/}}) likewise use efficient Hadamard transforms but drop the QJL sketch as harmful.
This motivates the normalized Hadamard transform as our practical TurboQuant-style fixed-transform baseline.
Concurrent with our work, OSCAR~\citep{zhou2026oscarofflinespectralcovarianceaware} estimates attention-aware covariance statistics offline, and constructs separate fixed \textit{orthogonal} transforms for keys and values.
It also calibrates per-layer clipping parameters that determine token-wise clipping thresholds.
The method is implemented in SGLang~\citep{NEURIPS2024_724be447}, a serving system compatible with paged KV caches.
OSCAR restricts its key and value transforms to be orthogonal, whereas WUSH-KV allows general invertible transforms.
This additional freedom permits anisotropic rescaling to jointly balance cache statistics and downstream sensitivity, which an orthogonal transform cannot generally realize.

\section{Background on the WUSH Transform}
\label{sec:wush}

Post-training quantization replaces a model's weights, and sometimes its activations, with low-precision codes that share a scale within each group of elements.
A handful of outliers carry far more amplitude than the rest, so they set the scale, and many of the available quantization levels go unused.
Transforms are one of the remedies: for a matrix $\bm{X}$ and an invertible transform matrix $\bm{T}$, quantize $\bm{T} \bm{X}$ rather than $\bm{X}$ itself, and undo the transform with $\bm{T}^{-1}$ afterwards.
What makes a transform good is not a property of $\bm{X}$ alone: the error only matters through the matrix product (e.g., $\bm{W}^{\top} \bm{X}$) in which $\bm{X}$ is one factor.
The WUSH transform~\citep{chen2026wush} therefore builds $\bm{T}$ in closed form from second-order statistics of both factors of that product.
According to their analysis, this choice is provably optimal for floating-point formats and asymptotically optimal for integer ones.
The transform is block diagonal, and its block width is the same as the quantization group size.
This keeps it cheap to apply to activations at inference time.
We follow the original WUSH construction and express it as a function of a Gram matrix and a Hessian, which is the form used throughout the rest of this paper.\footnote{The form we give is an effectively equivalent rewrite of the original construction, which leads to simpler notation for the KV-cache derivation.}

\textbf{The loss a transform has to control.}
Let $\bm{W} \in \mathbb{R}^{d \times d_{\mathrm{out}}}$ and $\bm{X} \in \mathbb{R}^{d \times d_{\mathrm{batch}}}$ be the weights and activations of a linear layer, restricted to the $d$ input coordinates one block acts on.
\citet{chen2026wush} quantizes both factors at once.
We will only ever quantize one of them, which is also the case to which their derivation reduces.
We take that factor to be $\bm{X}$, keep $\bm{W}$ exact, and let $\mathcal{Q}$ be a quantizer, specified in \cref{sec:method}.
The round trip through the invertible transform $\bm{T} \in \mathbb{R}^{d \times d}$ leaves behind a perturbation $\bm{\varepsilon} = \bm{T}^{-1} \mathcal{Q} \left( \bm{T} \bm{X} \right) - \bm{X}$.
With $\bm{H} = 2 \bm{W} \bm{W}^{\top}$, the loss this perturbation causes at the layer output is $\ell = \left\| \bm{W}^{\top} \bm{\varepsilon} \right\|_{\mathrm{F}}^{2} = 2^{-1} \sum_{j} \bm{\varepsilon}_{j}^{\top} \bm{H} \bm{\varepsilon}_{j}$, a quadratic form in the columns $\bm{\varepsilon}_{j}$ of $\bm{\varepsilon}$.
The matrix $\bm{H}$ is the Hessian of $\ell$ with respect to a column of $\bm{\varepsilon}$.

\textbf{The WUSH construction.}
Let $\bm{M} = \bm{X} \bm{X}^{\top}$ be the Gram matrix of $\bm{X}$, $\bm{\mathcal{H}} \in \left\{ \pm d^{-\frac{1}{2}} \right\}^{d \times d}$ a normalized (orthogonal) Hadamard matrix,\footnote{We assume throughout that a Hadamard matrix of order $d$ exists.} $\mathbf{I}$ the $d \times d$ identity, and $\gamma \ge 0$ a small damping ratio.
The transform $\bm{T}$ applied to $\bm{X}$ is built as
\begin{equation}
\begin{gathered}
\bm{L} \bm{L}^{\top} = \bm{H} + \gamma d^{-1} \operatorname{tr} \left(
\bm{H} \right) \mathbf{I}
, \qquad
\bm{U} \bm{\Lambda} \bm{U}^{\top} = \bm{L}^{\top} \left( \bm{M} + \gamma d^{-1}
\operatorname{tr} \left( \bm{M} \right) \mathbf{I} \right) \bm{L}
, \\
c = \left( 1 + \gamma \right)^{\frac{1}{2}} \left( \operatorname{tr} \left( \bm{M} \right) \right)^{\frac{1}{2}} \left(
\operatorname{tr} \left( \bm{\Lambda}^{\frac{1}{2}} \right) \right)^{-\frac{1}{2}}
, \qquad
\bm{T} = c\, \bm{\mathcal{H}} \bm{\Lambda}^{-\frac{1}{4}} \bm{U}^{\top} \bm{L}^{\top}
,
\end{gathered}
\label{eq:wush}
\end{equation}
where $\bm{L} \bm{L}^{\top}$ is the Cholesky decomposition and $\bm{U} \bm{\Lambda} \bm{U}^{\top}$ is the symmetric eigendecomposition.
The scalar $c$ is introduced so that $\left\| \bm{T} \bm{X} \right\|_{\mathrm{F}} = \left\| \bm{X} \right\|_{\mathrm{F}}$ when $\gamma = 0$, and approximately so otherwise.

\begin{definition}
\label{def:wush}
For symmetric positive semidefinite $\bm{M}, \bm{H} \in \mathbb{R}^{d \times d}$ and a damping ratio $\gamma \ge 0$ that makes both damped matrices positive definite, write $\textnormal{\textsc{Wush}} \left( \bm{M}, \bm{H}, \gamma \right)$ for the transform $\bm{T}$ that \cref{eq:wush} builds from them, in which $\bm{M}$ is the Gram matrix of the tensor being quantized and $\bm{H}$ is the Hessian of the loss with respect to that tensor's perturbation.
\end{definition}

We hold $\gamma$ fixed throughout and abbreviate $\textsc{Wush} \left( \bm{M}, \bm{H}, \gamma \right)$ to $\textsc{Wush} \left( \bm{M}, \bm{H} \right)$.
This construction is invariant to independent positive rescaling of its arguments, $\textsc{Wush} \left( \rho_{\mathrm{M}} \bm{M}, \rho_{\mathrm{H}} \bm{H} \right) = \textsc{Wush} \left( \bm{M}, \bm{H} \right)$ for every $\rho_{\mathrm{M}}, \rho_{\mathrm{H}} > 0$, so how the two matrices are normalized never has to be tracked.
What remains, for any tensor we wish to quantize, is to write down the Hessian of the loss with respect to its perturbation.

\section{WUSH-KV}
\label{sec:method}

This section specializes in WUSH within the key/value cache of grouped-query attention.
We first identify the cached tensors and formulate their transforms, then specify the quantizer and its guarantees, and finally describe cache management, storage, and computation costs.

\subsection{KV Cache in Grouped-Query Attention}

\textbf{Grouped-query attention (GQA).}
In a model of dimension $d_{\mathrm{model}}$, an attention module has $n_{\mathrm{q}}$ query heads and $n_{\mathrm{kv}}$ key/value heads of dimension $d$, so each key/value head is read by $n_{\mathrm{q}} / n_{\mathrm{kv}}$ query heads.
A query head is indexed by the pair $\left( h, g \right)$: the key/value head $h$ it reads and its position $g = 1, \dots, n_{\mathrm{q}} / n_{\mathrm{kv}}$ within that group.
Let $\bm{X} \in \mathbb{R}^{d_{\mathrm{model}} \times S}$ be the attention module input over a context of $S$ positions, one column each.
Each head projects it with $\bm{W}_{\mathrm{Q} \left( h, g \right)}, \bm{W}_{\mathrm{K} \left( h \right)}, \bm{W}_{\mathrm{V} \left( h \right)} \in \mathbb{R}^{d_{\mathrm{model}} \times d}$.
Let $\operatorname{Norm}$ denote any headwise normalization applied over the $d$ coordinates of each projected query or key.
This includes root mean square (RMS) normalization and reduces to the identity when no such normalization is used.
The query and key normalizations may have distinct parameters, which we suppress in the notation.
The rotary embedding $\operatorname{RoPE} \left( \bm{A} \right) = \left[ \bm{R}_{1} \bm{a}_{1}, \dots, \bm{R}_{S} \bm{a}_{S} \right]$ rotates each column by an orthogonal $\bm{R}_{s} \in \mathbb{R}^{d \times d}$ fixed by its position $s$.
Write $\tau$ for the softmax temperature and $\bm{\mathcal{M}} \in \left\{ 0, - \infty \right\}^{S \times S}$ for the causal mask, zero where a query is allowed to attend to a position and $- \infty$ elsewhere.
The output projection $\bm{W}_{\mathrm{O}} \in \mathbb{R}^{n_{\mathrm{q}} d \times d_{\mathrm{model}}}$ splits into one row block $\bm{W}_{\mathrm{O} \left( h, g \right)} \in \mathbb{R}^{d \times d_{\mathrm{model}}}$ per query head.
The attention module output is then
\begin{equation}
\begin{gathered}
\bm{Q}_{h,g} = \operatorname{RoPE} \left( \operatorname{Norm} \left( \bm{W}_{\mathrm{Q} \left( h, g \right)}^{\top} \bm{X} \right) \right)
, \qquad
\bm{K}_{h} = \operatorname{RoPE} \left( \operatorname{Norm} \left( \bm{W}_{\mathrm{K} \left( h \right)}^{\top} \bm{X} \right) \right)
, \\
\bm{V}_{h} = \bm{W}_{\mathrm{V} \left( h \right)}^{\top} \bm{X}
, \qquad
\bm{P}_{h,g} = \operatorname{softmax} \left( \tau^{-1} \bm{K}_{h}^{\top} \bm{Q}_{h,g} + \bm{\mathcal{M}} \right)
, \\
\bm{O}_{h,g} = \bm{V}_{h} \bm{P}_{h,g}
, \qquad
\bm{Y} = \sum_{h,g} \bm{W}_{\mathrm{O} \left( h, g \right)}^{\top} \bm{O}_{h,g}
,
\end{gathered}
\label{eq:attention}
\end{equation}
where the softmax acts on each column.

\textbf{Cached keys and values.}
Autoregressive generation grows the context one position at a time and re-evaluates \cref{eq:attention} at each new $S$.
Of the three projections $\bm{Q}_{h,g}$, $\bm{K}_{h}$, and $\bm{V}_{h}$, only the keys and values must persist because every later position attends to all earlier ones, whereas the queries are used once and discarded.
Each key/value head therefore caches $\bm{K}_{h} = \left[ \bm{k}_{h,1}, \dots, \bm{k}_{h,S} \right]$ and $\bm{V}_{h} = \left[ \bm{v}_{h,1}, \dots, \bm{v}_{h,S} \right]$ rather than recomputing them.
Each new token appends one column to both caches, and later queries reuse the previously cached columns unchanged.
Keys are stored after headwise normalization (when present) and rotary embedding, so $\bm{k}_{h,s}$ is read without further normalization or rotation.

\subsection{WUSH-KV Transforms}

\textbf{Transform placement.}
A cached key is read only through $\bm{K}_{h}^{\top} \bm{Q}_{h,g}$ and a cached value only through $\bm{W}_{\mathrm{O} \left( h, g \right)}^{\top} \bm{V}_{h}$, and neither product changes when an invertible matrix is inserted between its factors,
\begin{equation}
\bm{K}_{h}^{\top} \bm{Q}_{h,g} = \left( \bm{T}_{\mathrm{K} \left( h \right)} \bm{K}_{h} \right)^{\top} \left( \bm{T}_{\mathrm{K} \left( h \right)}^{-\top} \bm{Q}_{h,g} \right)
, \quad
\bm{W}_{\mathrm{O} \left( h, g \right)}^{\top} \bm{V}_{h} = \left( \bm{T}_{\mathrm{V} \left( h \right)}^{-\top} \bm{W}_{\mathrm{O} \left( h, g \right)} \right)^{\top} \left( \bm{T}_{\mathrm{V} \left( h \right)} \bm{V}_{h} \right)
.
\label{eq:insertion}
\end{equation}
With the quantizer of \cref{sec:quantizer}, we therefore cache $\mathcal{Q} \left( \bm{T}_{\mathrm{K} \left( h \right)} \bm{K}_{h} \right)$ and $\mathcal{Q} \left( \bm{T}_{\mathrm{V} \left( h \right)} \bm{V}_{h} \right)$ in place of the two tensors themselves.
For each key/value head, we build one $d \times d$ transform for the keys and another for the values.
The value-side transform is folded into the value and output projection weights offline before inference: we replace $\bm{W}_{\mathrm{V} \left( h \right)}$ with $\widehat{\bm{W}}_{\mathrm{V} \left( h \right)} = \bm{W}_{\mathrm{V} \left( h \right)} \bm{T}_{\mathrm{V} \left( h \right)}^{\top}$ and every $\bm{W}_{\mathrm{O} \left( h, g \right)}$ with $\widehat{\bm{W}}_{\mathrm{O} \left( h, g \right)} = \bm{T}_{\mathrm{V} \left( h \right)}^{-\top} \bm{W}_{\mathrm{O} \left( h, g \right)}$.
The key side does not disappear into the projection weights because the inserted transforms act after $\operatorname{Norm}$ and the position-dependent $\operatorname{RoPE}$, and a dense $\bm{T}_{\mathrm{K} \left( h \right)}$ cannot generally be moved across these operations.

This post-RoPE placement is deliberate.
As detailed in \cref{sec:alternative_key_placements}, fully folding the key transform and its query-side compensation into the projection weights requires each to commute with RoPE and its respective learned RMS normalization.
These joint commutation constraints leave only paired sign-flip transforms, which are trivial because they do not redistribute coordinate magnitudes.
A separate alternative is to place an unrestricted transform before RoPE rather than require it to commute with RoPE.
This retains the full $d \times d$ freedom, but its compensation depends on the cached position.
At each decoding step, this would require restoring and rotating every cached key or applying a different compensation at every cache position, adding position-dependent work across the cache and complicating the attention kernel.
We therefore retain the full $d \times d$ transform after RoPE, accepting fixed online transforms for each new key and query in exchange for a simple cache representation.

\textbf{Transform formulation.}
For the construction, we use local surrogate Hessians that retain the direct bilinear partner of each cached tensor: the queries of the group for a key and the output projection for a value.
The transforms are
\begin{equation}
\begin{gathered}
\bm{H}_{\mathrm{K} \left( h \right)} = 2 \sum_{g} \bm{Q}_{h,g} \bm{Q}_{h,g}^{\top}
, \qquad
\bm{H}_{\mathrm{V} \left( h \right)} = 2 \sum_{g} \bm{W}_{\mathrm{O} \left( h, g \right)} \bm{W}_{\mathrm{O} \left( h, g \right)}^{\top}
, \\
\bm{T}_{\mathrm{K} \left( h \right)} = \textsc{Wush} \left( \bm{K}_{h} \bm{K}_{h}^{\top}, \bm{H}_{\mathrm{K} \left( h \right)} \right)
, \qquad
\bm{T}_{\mathrm{V} \left( h \right)} = \textsc{Wush} \left( \bm{V}_{h} \bm{V}_{h}^{\top}, \bm{H}_{\mathrm{V} \left( h \right)} \right)
,
\end{gathered}
\label{eq:wushkv}
\end{equation}
with $\bm{K}_{h} \bm{K}_{h}^{\top}$ and $\bm{V}_{h} \bm{V}_{h}^{\top}$ the Gram matrices of the tensors being quantized.
\cref{alg:calibration} summarizes the offline calibration procedure for each attention module.
We accumulate statistics over all token positions in each calibration sequence, then sum them across sequences to avoid storing all calibration activations and projections at once.

\begin{algorithm}[H]
\caption{WUSH-KV calibration}
\label{alg:calibration}
\begin{algorithmic}[1]
\REQUIRE Weights $\bm{W}_{\mathrm{Q}}, \bm{W}_{\mathrm{K}}, \bm{W}_{\mathrm{V}}, \bm{W}_{\mathrm{O}}$, calibration sequences $\left\{ \bm{X}^{\left( i \right)} \right\}_{i = 1}^{n_{\mathrm{cal}}}$
\ENSURE Key transforms $\bm{T}_{\mathrm{K}}$ and folded weights $\widehat{\bm{W}}_{\mathrm{V}}, \widehat{\bm{W}}_{\mathrm{O}}$
\FOR{$i = 1, \dots, n_{\mathrm{cal}}$ independently}
\STATE $\bm{Q}^{\left( i \right)} \gets \operatorname{RoPE} \left( \operatorname{Norm} \left( \bm{W}_{\mathrm{Q}}^{\top} \bm{X}^{\left( i \right)} \right) \right); \quad \bm{K}^{\left( i \right)} \gets \operatorname{RoPE} \left( \operatorname{Norm} \left( \bm{W}_{\mathrm{K}}^{\top} \bm{X}^{\left( i \right)} \right) \right)$
\STATE $\bm{V}^{\left( i \right)} \gets \bm{W}_{\mathrm{V}}^{\top} \bm{X}^{\left( i \right)}$
\STATE $\bm{M}_{\mathrm{K} \left( h \right)}^{\left( i \right)} \gets \bm{K}_{h}^{\left( i \right)} \bm{K}_{h}^{\left( i \right) \top} \quad \forall\, h; \quad \bm{M}_{\mathrm{V} \left( h \right)}^{\left( i \right)} \gets \bm{V}_{h}^{\left( i \right)} \bm{V}_{h}^{\left( i \right) \top} \quad \forall\, h$
\STATE Compute Hessian matrices $\bm{H}_{\mathrm{K} \left( h \right)}^{\left( i \right)}, \bm{H}_{\mathrm{V} \left( h \right)}^{\left( i \right)}$ from $\bm{Q}^{\left( i \right)}, \bm{W}_{\mathrm{O}}$ and optionally $\bm{K}^{\left( i \right)}, \bm{V}^{\left( i \right)} \quad \forall\, h$
\ENDFOR
\STATE $\bm{T}_{\mathrm{K} \left( h \right)} \gets \textsc{Wush} \left( \sum_{i} \bm{M}_{\mathrm{K} \left( h \right)}^{\left( i \right)}, \sum_{i} \bm{H}_{\mathrm{K} \left( h \right)}^{\left( i \right)} \right) \quad \forall\, h$
\STATE $\bm{T}_{\mathrm{V} \left( h \right)} \gets \textsc{Wush} \left( \sum_{i} \bm{M}_{\mathrm{V} \left( h \right)}^{\left( i \right)}, \sum_{i} \bm{H}_{\mathrm{V} \left( h \right)}^{\left( i \right)} \right) \quad \forall\, h$
\STATE $\widehat{\bm{W}}_{\mathrm{V} \left( h \right)} \gets \bm{W}_{\mathrm{V} \left( h \right)} \bm{T}_{\mathrm{V} \left( h \right)}^{\top} \quad \forall\, h; \quad \widehat{\bm{W}}_{\mathrm{O} \left( h, g \right)} \gets \bm{T}_{\mathrm{V} \left( h \right)}^{-\top} \bm{W}_{\mathrm{O} \left( h, g \right)} \quad \forall\, h, g$
\end{algorithmic}
\end{algorithm}

\subsection{Quantizer and Near-Optimality}
\label{sec:quantizer}

We use per-token quantization: the $d$ channels of a key or value head form one group and share a scale.
To analyze clipping, we use the projection step of QuEST~\citep{pmlr-v267-panferov25a}.
For a nonzero group $\bm{x}$, it sets the step $\Delta$ from the group's root mean square (RMS) and reconstructs at bin centers,
\begin{equation}
\Delta = 2 \alpha_{b} \left( 2^{b} - 1 \right)^{-1} d^{-\frac{1}{2}} \left\| \bm{x} \right\|
, \qquad
\mathcal{Q} \left( \bm{x} \right) = \Delta \left( \left\lfloor \Delta^{-1} \bm{x} \right\rfloor + 2^{-1} \right)
,
\label{eq:quantizer}
\end{equation}
with the floor clamped to $\left[ - 2^{b-1}, 2^{b-1} - 1 \right]$ and $\mathcal{Q} \left( \bm{0} \right) = \bm{0}$.
The $2^{b}$ levels are evenly spaced between $\pm\alpha_{b}$ times the group RMS, and entries outside this range are clipped.
The constant $\alpha_{b}$ minimizes the mean squared error of the corresponding fixed-scale grid on a standard Gaussian.

For the analysis, write $\ell(\bm{T})$ for the quadratic output loss from \cref{sec:wush} after quantizing $\bm{T}\bm{X}$ and undoing $\bm{T}$.
We keep clipping exact and model rounding within the range by independent uniform noise, as specified in \cref{sec:noise_model}.
We call an invertible transform \emph{sensitivity-balanced} when $\bm{T}^{-\top} \bm{H} \bm{T}^{-1}$ has equal diagonal entries.

\begin{theorem}
\label{thm:near_optimal}
Let $\bm{M} = \bm{X} \bm{X}^{\top}$ and $\bm{H}$ be symmetric positive definite, and let $\bm{T}_{\star}$ be \cref{eq:wush} with $\gamma = 0$.
Under the Gaussian-tail and clipping-alignment conditions in \cref{sec:clipping}, with bounds independent of $d$ and $b$, every sensitivity-balanced invertible $\bm{T}$ satisfies
\begin{equation}
\mathbb{E} \left[ \ell \left( \bm{T}_{\star} \right) \right]
\le \left( 1 + O \left( \alpha_{b}^{-2} \right) \right)
\mathbb{E} \left[ \ell \left( \bm{T} \right) \right]
\qquad \text{as } b \to \infty
.
\label{eq:near_optimal}
\end{equation}
The expectation is over rounding noise, and the constant depends only on the assumption bounds.
\end{theorem}

The undamped WUSH transform is sensitivity-balanced by \cref{eq:equalization}.
\cref{sec:optimality} proves the result and gives an explicit finite-bit bound for $\alpha_{b} > 1$.
The guarantee concerns the ideal transform, while the experiments use damping.

\subsection{Full-Precision Cache Window}
\label{sec:cache_window}

We combine the quantized cache with two full-precision windows and quantize newly eligible entries in batches.
The first $S_{\mathrm{sink}}$ consecutive cache positions form a fixed full-precision sink window throughout generation.
Among the remaining positions, the newest $S_{\mathrm{keep}}$ form a rolling recent window.
Each new chunk is first appended in full precision and is used to compute its attention output.
We then quantize the largest multiple of $S_{\mathrm{flush}}$ among the entries that lie outside both windows under the current cache length.
Thus, $S_{\mathrm{sink}} + S_{\mathrm{keep}}$ entries per head remain persistently in full precision once the two windows no longer overlap, and fewer than $S_{\mathrm{flush}}$ additional entries can remain temporarily in full precision.
Write $S_{\mathrm{quant}}$ for the inclusive right boundary of the quantized middle region, initialized to zero for an empty cache.
\cref{alg:cache_management} gives this update for one attention module (some symbols are overloaded for transformed-coordinate counterparts).

\begin{algorithm}[H]
\caption{WUSH-KV inference}
\label{alg:cache_management}
\begin{algorithmic}[1]
\REQUIRE Weights $\bm{W}_{\mathrm{Q}}, \bm{W}_{\mathrm{K}}, \widehat{\bm{W}}_{\mathrm{V}}, \widehat{\bm{W}}_{\mathrm{O}}$, transforms $\bm{T}_{\mathrm{K}}$, cache $\bm{K}, \bm{V}$ in transformed coordinates with length $S$ and quantized boundary $S_{\mathrm{quant}}$, new activation chunk $\bm{X}'$ of length $S'$, window sizes $S_{\mathrm{sink}}, S_{\mathrm{keep}}$ with flush size $S_{\mathrm{flush}}$
\ENSURE Attention output $\bm{Y}'$, updated cache $\bm{K}, \bm{V}$ of length $S$ and quantized boundary $S_{\mathrm{quant}}$
\STATE $\bm{Q}' \gets \operatorname{RoPE} \left( \operatorname{Norm} \left( \bm{W}_{\mathrm{Q}}^{\top} \bm{X}' \right) \right); \quad \bm{K}' \gets \operatorname{RoPE} \left( \operatorname{Norm} \left( \bm{W}_{\mathrm{K}}^{\top} \bm{X}' \right) \right)$
\STATE $\bm{V}' \gets \widehat{\bm{W}}_{\mathrm{V}}^{\top} \bm{X}'$
\STATE $\bm{Q}_{h,g}' \gets \bm{T}_{\mathrm{K} \left( h \right)}^{-\top} \bm{Q}_{h,g}' \quad \forall\, h, g; \quad \bm{K}_{h}' \gets \bm{T}_{\mathrm{K} \left( h \right)} \bm{K}_{h}' \quad \forall\, h$
\STATE Append $\bm{K}', \bm{V}'$ to $\bm{K}, \bm{V}$ in full precision; \quad $S \gets S + S'$
\STATE Compute $\bm{O}'$ from $\bm{Q}', \bm{K}, \bm{V}$ using grouped-query attention (mixed precision)
\STATE $\bm{Y}' \gets \widehat{\bm{W}}_{\mathrm{O}}^{\top} \bm{O}'$
\STATE $S_{\mathrm{start}} \gets \max \left( S_{\mathrm{quant}}, S_{\mathrm{sink}} \right)$
\STATE $S_{\mathrm{quant}} \gets \max \left( S_{\mathrm{quant}}, S_{\mathrm{start}} + \left\lfloor \left( S - S_{\mathrm{keep}} - S_{\mathrm{start}} \right) / S_{\mathrm{flush}} \right\rfloor S_{\mathrm{flush}} \right)$
\IF{$S_{\mathrm{quant}} > S_{\mathrm{start}}$}
\STATE $\bm{k}_{h,s} \gets \mathcal{Q} \left( \bm{k}_{h,s} \right) \quad \forall\, h,\ S_{\mathrm{start}} < s \le S_{\mathrm{quant}}$
\STATE $\bm{v}_{h,s} \gets \mathcal{Q} \left( \bm{v}_{h,s} \right) \quad \forall\, h,\ S_{\mathrm{start}} < s \le S_{\mathrm{quant}}$
\ENDIF
\end{algorithmic}
\end{algorithm}

\subsection{Storage and Computation Costs}

For a model with $n_{\mathrm{layer}}$ attention modules, at sequence length $S$, the model caches $2 n_{\mathrm{layer}} n_{\mathrm{kv}} d S$ key/value elements.
Because the full-precision windows typically cover less than 1\% of the maximum sequence length in the downstream tasks, they increase the effective cache quantization bitwidth only slightly.
The method stores $n_{\mathrm{layer}} n_{\mathrm{kv}}$ key transforms of size $d \times d$, fixed once at calibration time and shared by every position.
Their storage overhead is small relative to both the weights and the KV cache at practical sequence lengths.
The value-side transform adds no online cost because both halves are folded into the weights.
On the key side, each new key costs one $d \times d$ product, and so does each query, or $d^{2}$ multiply-accumulates per position for each of the $n_{\mathrm{kv}}$ key heads and the $n_{\mathrm{q}}$ query heads.

\section{Experiments}
\label{sec:experiments}

We evaluate WUSH-KV at three levels: controlled attention-stage reconstruction error, WikiText-2~\citep{merity2017pointer} perplexity, and downstream reasoning accuracy.

\subsection{Attention-Stage Quantization Error}
\label{sec:experiments_attention_error}

End-to-end metrics do not isolate the contribution of the KV transforms.
We therefore conduct a controlled ablation of the transform choices, measuring attention-stage reconstruction error with keys and values quantized separately and jointly.

\textbf{Setup.}
We study all 36 attention modules of Qwen3-8B.
We calibrate on 128 FineWeb-Edu~\citep{NEURIPS2024_370df50c} sequences and evaluate on 32 disjoint sequences, each of length 1024.
We evaluate reconstruction at 2, 3, and 4 bits in a prefill-like setting.
Every transform uses the QuEST quantizer in \cref{eq:quantizer}, which holds the quantization rule fixed and isolates the effect of the transform.
The entire KV sequence is quantized without the full-precision cache windows in \cref{sec:cache_window}, which isolates the effect of the transforms on reconstruction error.
Reference activations follow the clean BF16 model trajectory.

\textbf{Transforms.}
We compare six key/value transform pairs: identity (I), random orthogonal (R), normalized Hadamard (H), OSCAR~\citep{zhou2026oscarofflinespectralcovarianceaware}, WUSH-KV (WUSH), and an attention-aware WUSH-KV variant (WUSH-A).
This H choice reflects the practical TurboQuant-style implementations.
H does not reproduce the complete TurboQuant codec because every method in this experiment uses the same QuEST quantizer, isolating the effect of the transform.
WUSH, our proposed method, learns separate $\bm{T}_{\mathrm{K}}$ and $\bm{T}_{\mathrm{V}}$ from the Gram matrices and simple Hessians in \cref{eq:wushkv}.
We use $\gamma = 10^{-2}$ for WUSH.
WUSH-A uses the same construction and damping but replaces the simple Hessians with the attention-aware Hessians derived in \cref{sec:exact_gqa_hessians}.
The attention-aware Hessians require only forward-pass quantities, so WUSH-A needs neither backpropagation nor Hessian estimation through automatic differentiation.
An empirical-Fisher alternative could use language-model-loss gradients, but we do not evaluate it because backward-pass calibration is too expensive.

\begin{figure}[!htb]
\begin{center}
\includegraphics[width=0.7\linewidth]{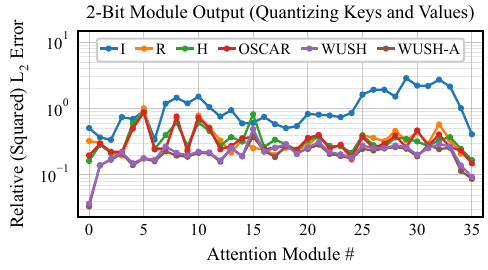}
\caption{Relative (squared) $\mathrm{L}_{2}$ module-output errors for I, R, H, OSCAR, WUSH, and WUSH-A with both keys and values quantized at 2-bit.}
\label{fig:attention_stage_int2}
\end{center}
\end{figure}

\textbf{Results.}
For each module and measured quantity $\bm{A}$, we report $\lVert \widehat{\bm{A}} - \bm{A} \rVert_{\mathrm{F}}^{2} / \lVert \bm{A} \rVert_{\mathrm{F}}^{2}$, summing the numerator and denominator separately over evaluation sequences.
\cref{fig:attention_stage_int2} measures module-output error after the output projection and before the residual connection, with both keys and values quantized.
The largest layerwise reductions occur in the first few attention modules.
At 2-bit, when both keys and values are quantized, the geometric-mean module-output errors are $0.208$ for WUSH and $0.195$ for WUSH-A, compared with $0.325$ for H and $0.309$ for OSCAR.
In \cref{sec:additional_attention_stage_errors}, the additional 2-bit plots in \cref{fig:attention_stage_diagnostics} show that WUSH reduces error in the query-key dot products, softmax probabilities, and module output when only keys are quantized, as well as module-output error when only values are quantized.
The 3-bit and 4-bit plots appear in \cref{fig:attention_stage_combined}.
Across bitwidths, WUSH and WUSH-A achieve lower module-output error than H and OSCAR for most attention modules in this setting and provide the strongest overall results among the evaluated transforms.
WUSH-A offers only a modest improvement over WUSH here.
We therefore use the simpler WUSH Hessians for the main method.
WUSH-A remains a practical alternative for future models or settings in which attention-aware sensitivity provides a larger gain.

\subsection{Perplexity Evaluations}
\label{sec:experiments_perplexity}

We first test whether the lower attention-stage reconstruction error of WUSH translates into better language-model quality by measuring Qwen3-8B perplexity on WikiText-2.

\begin{wraptable}{r}{0.35\textwidth}
\caption{Qwen3-8B WikiText-2 perplexity.
The full-precision baseline is $9.72$.}
\label{tab:wikitext_perplexity}
\begin{center}
\begin{tabular}{@{}c|ccc@{}}
\toprule
Transform & 4-Bit & 3-Bit & 2-Bit \\
\midrule
I & 11.36 & 15.29 & 38.08 \\
R & 9.77 & 10.18 & 14.22 \\
H & 9.77 & 10.06 & 258.29 \\
OSCAR & 9.77 & 10.19 & 13.74 \\
WUSH & \textbf{9.73} & \textbf{9.82} & \textbf{10.51} \\
\bottomrule
\end{tabular}
\end{center}
\end{wraptable}

\textbf{Setup.}
For WUSH and OSCAR calibration, we use 128 FineWeb-Edu sequences of length 32,768.
We use YaRN with factor $4$ for calibrations and evaluations on Qwen3-8B.
For WUSH, following \cref{alg:calibration}, an unquantized forward pass accumulates key Gram matrices after RMS normalization and RoPE, value Gram matrices, and simple per-head Hessians defined in \cref{eq:wushkv}.
We then construct one key and one value transform for every attention module and KV head with $\gamma = 10^{-2}$.
Calibration is performed once, and the resulting transforms are reused across all cache bitwidths.
On a single NVIDIA L40S GPU, this one-time Qwen3-8B calibration takes about 12 minutes and peaks at 39 GiB of GPU memory.
All quantized runs use the QuEST quantizer in \cref{eq:quantizer}, matching our theoretical analysis.
We compare identity (I), random orthogonal (R), normalized Hadamard (H), OSCAR, and WUSH-KV (WUSH) at 4, 3, and 2 bits.
For evaluation, we concatenate and tokenize the WikiText-2 test split, then divide it into non-overlapping sequences of length 2,048.
Each sequence starts with an empty cache and is processed in 16-token forward chunks.
Perplexity is computed from the average next-token negative log-likelihood.
For the quantized runs, we use the full-precision cache windows in \cref{sec:cache_window} with $S_{\mathrm{sink}} = 16$, $S_{\mathrm{keep}} = 128$, and $S_{\mathrm{flush}} = 16$.

\textbf{Results.}
\cref{tab:wikitext_perplexity} shows that WUSH achieves the lowest perplexity among the quantized transforms at every bitwidth.
OSCAR remains competitive at higher bitwidths, but WUSH's advantage grows under more aggressive quantization: at 2-bit, WUSH obtains $10.51$ perplexity compared with OSCAR's $13.74$. 
WUSH otherwise remains close to the full-precision baseline, while the other transforms degrade more sharply as the bitwidth decreases.  
The extreme 2-bit perplexity for H arises from spikes in the first attention module's key vectors that are concentrated in essentially one coordinate before the transform.
The Hadamard transform spreads these spikes into nearly equal-magnitude coefficients close to the midpoints between adjacent 2-bit reconstruction levels, causing dense quantization error.

\subsection{Downstream Benchmarks}
\label{sec:experiments_benchmarks}

We next evaluate whether WUSH-KV's relative advantage carries over to downstream reasoning accuracy.
For these end-to-end evaluations, we integrate WUSH-KV into SGLang with the OSCAR quantizer~\citep{zhou2026oscarofflinespectralcovarianceaware}.
This quantizer combines percentile clipping with asymmetric affine scaling and therefore does not readily admit the exact same proof as QuEST.
Its adaptive zero point can use the available quantization levels more efficiently for shifted or asymmetric groups, while percentile clipping suppresses rare large errors.
In our experiments, this tradeoff yields better end-to-end quality than symmetric QuEST, despite slightly higher average local squared error.
Using the same quantizer also enables a direct comparison with the concurrent OSCAR method.

\textbf{Setup.}
We evaluate Qwen3-4B-Thinking-2507, Qwen3-8B, and Qwen3-32B on AIME 2025~\citep{aime25}, MATH-500~\citep{ICLR2024_aca97732}, GPQA Diamond~\citep{rein2024gpqa}, and LiveCodeBench v6~\citep{jain2025livecodebench} with stochastic sampling over three seeds.
We calibrate WUSH separately for each model on FineWeb-Edu, as in the perplexity evaluation, and reuse its transforms across downstream tasks.
We use YaRN with factor $4$ for calibrations and evaluations on Qwen3-8B and Qwen3-32B, and keep it disabled on Qwen3-4B-Thinking-2507.
For the OSCAR baseline, we use its \emph{released} GPQA Diamond calibrated transforms but were unable to exactly reproduce all scores reported in the concurrent preprint~\citep{zhou2026oscarofflinespectralcovarianceaware}, so we report our runs under the same SGLang evaluation pipeline.
For a transformed token-head group $\bm{x} \in \mathbb{R}^{d}$ at bitwidth $b$, the OSCAR-style quantizer clips the coordinates at their $\kappa$-quantile, derives an affine scale $\Delta$ and zero point $z$ from the clipped range, rounds, and reconstructs,
\begin{equation}
\begin{gathered}
x_{\min} = \max \left\{ \min x_{i}, -\operatorname{quantile}_{\kappa} \left( \left| x_{i} \right| \right) \right\},
\quad
x_{\max} = \min \left\{ \max x_{i}, \operatorname{quantile}_{\kappa} \left( \left| x_{i} \right| \right) \right\},
\\
\Delta = \left( 2^{b} - 1 \right)^{-1} \left( x_{\max} - x_{\min} \right),
\quad
z = -\Delta^{-1} x_{\min},
\quad
\mathcal{Q} \left( \bm{x} \right) = \Delta \left( \left\lfloor \Delta^{-1} \bm{x} + z \bm{1} \right\rceil - z \bm{1} \right).
\end{gathered}
\label{eq:oscar_quantizer}
\end{equation}
with $\operatorname{quantile}_{\kappa} \left( \left| x_i \right| \right)$ taken over $i = 1, \ldots, d$ and $\left\lfloor \cdot \right\rceil$ denoting round-to-nearest clamped to $\left[ 0, 2^{b} - 1 \right]$.
We use $\kappa_{\mathrm{K}} = 0.96$ for all three models and $\kappa_{\mathrm{V}} = 0.92$ for Qwen3-4B-Thinking-2507 and Qwen3-8B, increasing $\kappa_{\mathrm{V}}$ to $0.96$ for Qwen3-32B.
We reuse OSCAR's tuned clipping ratios for WUSH-KV without retuning them for the WUSH-transformed K/V distributions.
These ratios may therefore be suboptimal for WUSH-KV, leaving potential headroom from WUSH-specific clipping optimization.
We provide a quantizer ablation and calibration controls in \cref{sec:quantizer_calibration_choices}.
All quantized benchmark runs use $S_{\mathrm{sink}} = 64$, $S_{\mathrm{keep}} = 256$, and $S_{\mathrm{flush}} = 8$, matching OSCAR's BF16 sink and recent windows with page-sized INT2 flushing.
The OSCAR SGLang integration updates cache storage before attention, while still using the current token or prefill chunk's K/V in BF16.

\begin{table}[!htb]
\caption{Downstream results with 2-bit KV caches.
Each cell gives the task score (\%) with its standard deviation over three seeds.
The gray second line gives the mean generated-token count in thousands (k) and the budget-exhaustion rate.
Bold marks the best quantized mean score and the shortest mean output for each model and task.}
\label{tab:downstream_benchmarks}
\renewcommand{\arraystretch}{2.0}
\newcommand{\benchmarkcell}[4]{\shortstack[c]{\makebox[0.86in][l]{#1\,$\pm$\,#2}\\[-1pt]\makebox[0.86in][r]{\textcolor{black!65}{#3\enspace#4\%}}}}
\newcommand{\hadamardlabel}{\shortstack[c]{TurboQuant-Style\\\textit{H Transform}}}
\newcommand{\wushkvlabel}{\shortstack[c]{WUSH-KV\\\textit{WUSH Transform}}}
\begin{center}
\begin{tabular}{@{}c|c|cccc@{}}
\toprule
\multirow{1}{*}{Model} & \multirow{1}{*}{Method} & \multirow{1}{*}{AIME 2025} & \multirow{1}{*}{MATH-500} & \multirow{1}{*}{GPQA Diamond} & \multirow{1}{*}{LiveCodeBench v6} \\
\midrule
\multirow{4}{*}{\rotatebox[origin=c]{90}{\shortstack[c]{Qwen3-4B-\\Thinking-2507}}}
& \multirow{1}{*}{BF16} & \benchmarkcell{77.8}{5.1}{21.8k}{0} & \benchmarkcell{89.5}{1.5}{5.8k}{0} & \benchmarkcell{69.9}{2.0}{9.0k}{0} & \benchmarkcell{57.1}{0.6}{19.2k}{0} \\
& \multirow{1}{*}{\hadamardlabel} & \benchmarkcell{0.0}{0.0}{91.8k}{100} & \benchmarkcell{18.7}{1.0}{52.9k}{76} & \benchmarkcell{16.2}{1.8}{56.4k}{68} & \benchmarkcell{0.0}{0.0}{91.7k}{99} \\
& \multirow{1}{*}{OSCAR} & \benchmarkcell{65.6}{1.9}{\textbf{24.2k}}{0} & \benchmarkcell{89.3}{0.5}{6.6k}{0} & \benchmarkcell{65.3}{2.8}{10.1k}{1} & \benchmarkcell{50.3}{4.5}{25.0k}{2} \\
& \multirow{1}{*}{\wushkvlabel}
& \benchmarkcell{\textbf{71.1}}{6.9}{24.3k}{0}
& \benchmarkcell{\textbf{90.1}}{0.5}{\textbf{6.6k}}{0}
& \benchmarkcell{\textbf{66.3}}{2.8}{\textbf{10.0k}}{0}
& \benchmarkcell{\textbf{52.4}}{0.3}{\textbf{23.0k}}{0} \\
\midrule
\multirow{4}{*}{\rotatebox[origin=c]{90}{Qwen3-8B}}
& \multirow{1}{*}{BF16} & \benchmarkcell{67.8}{1.9}{45.4k}{11} & \benchmarkcell{95.1}{0.3}{10.8k}{1} & \benchmarkcell{57.4}{3.8}{16.5k}{1} & \benchmarkcell{47.8}{2.2}{45.7k}{11} \\
& \multirow{1}{*}{\hadamardlabel} & \benchmarkcell{0.0}{0.0}{92.1k}{98} & \benchmarkcell{21.9}{0.1}{52.3k}{75} & \benchmarkcell{30.1}{2.8}{28.6k}{29} & \benchmarkcell{1.0}{0.3}{87.8k}{89} \\
& \multirow{1}{*}{OSCAR} & \benchmarkcell{34.4}{8.4}{63.0k}{42} & \benchmarkcell{89.3}{1.1}{16.0k}{6} & \benchmarkcell{52.9}{3.0}{18.3k}{1} & \benchmarkcell{20.4}{0.3}{73.8k}{71} \\
& \multirow{1}{*}{\wushkvlabel} & \benchmarkcell{\textbf{46.7}}{3.3}{\textbf{53.4k}}{12} & \benchmarkcell{\textbf{92.9}}{0.1}{\textbf{12.2k}}{2} & \benchmarkcell{\textbf{54.5}}{2.3}{\textbf{16.3k}}{1} & \benchmarkcell{\textbf{35.2}}{0.7}{\textbf{53.9k}}{26} \\
\midrule
\multirow{4}{*}{\rotatebox[origin=c]{90}{Qwen3-32B}}
& \multirow{1}{*}{BF16} & \benchmarkcell{74.4}{5.1}{37.1k}{6} & \benchmarkcell{94.7}{1.2}{7.1k}{1} & \benchmarkcell{66.7}{2.7}{10.5k}{0} & \benchmarkcell{61.9}{0.3}{40.9k}{3} \\
& \multirow{1}{*}{\hadamardlabel} & \benchmarkcell{1.1}{1.9}{89.2k}{91} & \benchmarkcell{0.0}{0.0}{65.5k}{100} & \benchmarkcell{0.0}{0.0}{65.5k}{100} & \benchmarkcell{1.0}{1.2}{80.2k}{73} \\
& \multirow{1}{*}{OSCAR} & \benchmarkcell{\textbf{63.3}}{3.3}{\textbf{47.2k}}{9} & \benchmarkcell{\textbf{93.1}}{0.7}{\textbf{9.7k}}{2} & \benchmarkcell{\textbf{60.6}}{3.0}{13.6k}{2} & \benchmarkcell{27.8}{1.4}{71.4k}{53} \\
& \multirow{1}{*}{\wushkvlabel} & \benchmarkcell{60.0}{5.8}{47.7k}{9} & \benchmarkcell{91.9}{1.4}{10.6k}{2} & \benchmarkcell{59.3}{2.0}{\textbf{12.1k}}{4} & \benchmarkcell{\textbf{43.1}}{2.3}{\textbf{54.3k}}{9} \\
\bottomrule
\end{tabular}
\end{center}
\end{table}

\textbf{Results.}
\cref{tab:downstream_benchmarks} reports accuracy and generation behavior across model sizes.
The budget-exhaustion rate is the fraction of samples that reach the task-specific generation limit without a valid answer.
The limit is 65,536 tokens for MATH-500 and GPQA Diamond and 92,160 tokens for AIME 2025 and LiveCodeBench v6.
For a strong direct comparison, we use OSCAR's selected 2-bit quantizer and cache policy without tuning them specifically for WUSH-KV.
Even in this OSCAR-favored setup, WUSH-KV and OSCAR are close when OSCAR maintains strong accuracy, including across the 4B tasks and most 32B tasks.
At 8B, WUSH-KV scores higher than OSCAR on all four benchmarks.
When OSCAR degrades sharply, WUSH-KV can be substantially better, as on LiveCodeBench v6 for both 8B and 32B.
Relative to BF16, WUSH-KV's degradation is comparatively consistent across model sizes, whereas OSCAR's varies more by model and task.
WUSH-KV generates at most about $1.5\times$ as many tokens as BF16 in these tasks, well below the nominal $8\times$ reduction in storage per quantized KV entry from BF16 to INT2.
Thus, the longer outputs do not erase the expected cache-memory benefit.
Hadamard's large score losses and frequent budget exhaustion show that this fixed transform alone is insufficient for reliable 2-bit KV-cache quantization across these tasks.

We additionally evaluate long-context RULER NIAH~\citep{hsieh2024ruler} and MRCR~\citep{vodrahalli2024michelangelolongcontextevaluations}, finding that WUSH-KV retains higher accuracy than OSCAR at the longest tested lengths.
The detailed results are presented in \cref{sec:long_context_evaluations}.

\section{Conclusion}
\label{sec:conclusion}

We presented WUSH-KV for low-bit KV-cache quantization in grouped-query attention.
The method uses calibration data to construct WUSH transforms that are applied to cached keys and values before quantization.
These transforms can be paired with per-token clipped quantizers.
For QuEST, our analysis provides a clipping-aware near-optimality guarantee among sensitivity-balanced transforms under the stated noise model and assumptions.
Our end-to-end downstream benchmarks pair the transforms with an OSCAR-style percentile-clipped affine quantizer.
The results show that the WUSH transform can improve the quality of low-bit KV cache quantization.
One limitation is that the dense key-side transform must remain online because it is applied after RoPE, which adds extra computation during inference.
The value-side transform can be folded into the model weights, though.
WUSH-KV supports the use of a more general calibration Hessian to construct the transforms.
Future work could study cheaper structured transforms, stronger attention-aware sensitivity measures, and system-level optimization and throughput evaluation.

\section*{Acknowledgements}
\addcontentsline{toc}{section}{Acknowledgements}

This research was funded in part by the Austrian Science Fund (FWF) 10.55776/COE12 and was partially supported by a generous grant from NVIDIA.
The authors would like to thank Verda Cloud for computational support, and in particular, Paul Chang for his consistent, prompt, and generous help throughout the project.

\bibliography{bibliography}

@inproceedings{
chen2026wush,
title={{WUSH}: Near-Optimal Adaptive Transforms for {LLM} Quantization},
author={Jiale Chen and Vage Egiazarian and Roberto L. Castro and Torsten Hoefler and Dan Alistarh},
booktitle={Forty-third International Conference on Machine Learning},
year={2026},
url={https://openreview.net/forum?id=ZsECxUkbKB}
}

@InProceedings{pmlr-v267-panferov25a,
  title = 	 {{Q}u{EST}: Stable Training of {LLM}s with 1-Bit Weights and Activations},
  author =       {Panferov, Andrei and Chen, Jiale and Tabesh, Soroush and Nikdan, Mahdi and Alistarh, Dan},
  booktitle = 	 {Proceedings of the 42nd International Conference on Machine Learning},
  pages = 	 {47820--47836},
  year = 	 {2025},
  editor = 	 {Singh, Aarti and Fazel, Maryam and Hsu, Daniel and Lacoste-Julien, Simon and Berkenkamp, Felix and Maharaj, Tegan and Wagstaff, Kiri and Zhu, Jerry},
  volume = 	 {267},
  series = 	 {Proceedings of Machine Learning Research},
  month = 	 {13--19 Jul},
  publisher =    {PMLR},
  url = 	 {https://proceedings.mlr.press/v267/panferov25a.html}
}

@inproceedings{NEURIPS2024_b5b93943,
 author = {Ashkboos, Saleh and Mohtashami, Amirkeivan and Croci, Maximilian L. and Li, Bo and Cameron, Pashmina and Jaggi, Martin and Alistarh, Dan and Hoefler, Torsten and Hensman, James},
 booktitle = {Advances in Neural Information Processing Systems},
 doi = {10.52202/079017-3180},
 editor = {A. Globerson and L. Mackey and D. Belgrave and A. Fan and U. Paquet and J. Tomczak and C. Zhang},
 pages = {100213--100240},
 publisher = {Curran Associates, Inc.},
 title = {{Q}ua{R}ot: Outlier-Free 4-Bit Inference in Rotated {LLM}s},
 url = {https://proceedings.neurips.cc/paper_files/paper/2024/file/b5b939436789f76f08b9d0da5e81af7c-Paper-Conference.pdf},
 volume = {37},
 year = {2024}
}

@InProceedings{pmlr-v235-liu24bz,
  title = 	 {{KIVI}: A Tuning-Free Asymmetric 2bit Quantization for {KV} Cache},
  author =       {Liu, Zirui and Yuan, Jiayi and Jin, Hongye and Zhong, Shaochen and Xu, Zhaozhuo and Braverman, Vladimir and Chen, Beidi and Hu, Xia},
  booktitle = 	 {Proceedings of the 41st International Conference on Machine Learning},
  pages = 	 {32332--32344},
  year = 	 {2024},
  editor = 	 {Salakhutdinov, Ruslan and Kolter, Zico and Heller, Katherine and Weller, Adrian and Oliver, Nuria and Scarlett, Jonathan and Berkenkamp, Felix},
  volume = 	 {235},
  series = 	 {Proceedings of Machine Learning Research},
  month = 	 {21--27 Jul},
  publisher =    {PMLR},
  url = 	 {https://proceedings.mlr.press/v235/liu24bz.html}
}

@inproceedings{NEURIPS2024_028fcbcf,
 author = {Hooper, Coleman and Kim, Sehoon and Mohammadzadeh, Hiva and Mahoney, Michael W. and Shao, Yakun Sophia and Keutzer, Kurt and Gholami, Amir},
 booktitle = {Advances in Neural Information Processing Systems},
 doi = {10.52202/079017-0040},
 editor = {A. Globerson and L. Mackey and D. Belgrave and A. Fan and U. Paquet and J. Tomczak and C. Zhang},
 pages = {1270--1303},
 publisher = {Curran Associates, Inc.},
 title = {{KVQ}uant: Towards 10 Million Context Length {LLM} Inference with {KV} Cache Quantization},
 url = {https://proceedings.neurips.cc/paper_files/paper/2024/file/028fcbcf85435d39a40c4d61b42c99a4-Paper-Conference.pdf},
 volume = {37},
 year = {2024}
}

@inproceedings{ICLR2025_e5b1c0d4,
 author = {Liu, Zechun and Zhao, Changsheng and Fedorov, Igor and Soran, Bilge and Choudhary, Dhruv and Krishnamoorthi, Raghuraman and Chandra, Vikas and Tian, Yuandong and Blankevoort, Tijmen},
 booktitle = {International Conference on Learning Representations},
 editor = {Y. Yue and A. Garg and N. Peng and F. Sha and R. Yu},
 pages = {92009--92032},
 title = {{S}pin{Q}uant: {LLM} Quantization with Learned Rotations},
 url = {https://proceedings.iclr.cc/paper_files/paper/2025/file/e5b1c0d4866f72393c522c8a00eed4eb-Paper-Conference.pdf},
 volume = {2025},
 year = {2025}
}

@InProceedings{pmlr-v267-sun25l,
  title = 	 {{F}lat{Q}uant: Flatness Matters for {LLM} Quantization},
  author =       {Sun, Yuxuan and Liu, Ruikang and Bai, Haoli and Bao, Han and Zhao, Kang and Li, Yuening and Hu, Jiaxin and Yu, Xianzhi and Hou, Lu and Yuan, Chun and Jiang, Xin and Liu, Wulong and Yao, Jun},
  booktitle = 	 {Proceedings of the 42nd International Conference on Machine Learning},
  pages = 	 {57587--57613},
  year = 	 {2025},
  editor = 	 {Singh, Aarti and Fazel, Maryam and Hsu, Daniel and Lacoste-Julien, Simon and Berkenkamp, Felix and Maharaj, Tegan and Wagstaff, Kiri and Zhu, Jerry},
  volume = 	 {267},
  series = 	 {Proceedings of Machine Learning Research},
  month = 	 {13--19 Jul},
  publisher =    {PMLR},
  url = 	 {https://proceedings.mlr.press/v267/sun25l.html}
}

@inproceedings{ijcai2025p690,
  title     = {{R}otate{KV}: Accurate and Robust 2-Bit {KV} Cache Quantization for {LLM}s via Outlier-Aware Adaptive Rotations },
  author    = {Su, Zunhai and Wei, Hanyu and Chen, Zhe and Shen, Wang and Li, Linge and Yu, Huangqi and Yuan, Kehong},
  booktitle = {Proceedings of the Thirty-Fourth International Joint Conference on
               Artificial Intelligence, {IJCAI-25}},
  publisher = {International Joint Conferences on Artificial Intelligence Organization},
  editor    = {James Kwok},
  pages     = {6200--6208},
  year      = {2025},
  month     = {8},
  note      = {Main Track},
  doi       = {10.24963/ijcai.2025/690},
  url       = {https://www.ijcai.org/proceedings/2025/690},
}

@inproceedings{ICLR2026_5c802ef3,
 author = {Zandieh, Amir and Daliri, Majid and Hadian, Majid and Mirrokni, Vahab},
 booktitle = {International Conference on Learning Representations},
 editor = {C. Vondrick and B. Hariharan and C. Raffel and L. Pinto and D. Yang and A. Faust},
 pages = {56418--56439},
 title = {{T}urbo{Q}uant: Online Vector Quantization with Near-optimal Distortion Rate},
 url = {https://proceedings.iclr.cc/paper_files/paper/2026/file/5c802ef38ab6e366c2ea06eee554c088-Paper-Conference.pdf},
 volume = {2026},
 year = {2026}
}

@misc{zhou2026oscarofflinespectralcovarianceaware,
      title={{OSCAR}: Offline Spectral Covariance-Aware Rotation for 2-bit {KV} Cache Quantization},
      author={Zhongzhu Zhou and Donglin Zhuang and Jisen Li and Ziyan Chen and Shuaiwen Leon Song and Ben Athiwaratkun and Xiaoxia Wu},
      year={2026},
      eprint={2605.17757},
      archivePrefix={arXiv},
      primaryClass={cs.LG},
      url={https://arxiv.org/abs/2605.17757}, 
}

@misc{benbasat2026noteturboquantearlierdriveeden,
      title={A Note on {T}urbo{Q}uant and the Earlier {DRIVE}/{EDEN} Line of Work},
      author={Ran Ben-Basat and Yaniv Ben-Itzhak and Gal Mendelson and Michael Mitzenmacher and Amit Portnoy and Shay Vargaftik},
      year={2026},
      eprint={2604.18555},
      archivePrefix={arXiv},
      primaryClass={cs.LG},
      url={https://arxiv.org/abs/2604.18555}, 
}

@inproceedings{NEURIPS2024_724be447,
 author = {Zheng, Lianmin and Yin, Liangsheng and Xie, Zhiqiang and Sun, Chuyue and Huang, Jeff and Yu, Cody Hao and Cao, Shiyi and Kozyrakis, Christos and Stoica, Ion and Gonzalez, Joseph E. and Barrett, Clark and Sheng, Ying},
 booktitle = {Advances in Neural Information Processing Systems},
 doi = {10.52202/079017-2000},
 editor = {A. Globerson and L. Mackey and D. Belgrave and A. Fan and U. Paquet and J. Tomczak and C. Zhang},
 pages = {62557--62583},
 publisher = {Curran Associates, Inc.},
 title = {{SGL}ang: Efficient Execution of Structured Language Model Programs},
 url = {https://proceedings.neurips.cc/paper_files/paper/2024/file/724be4472168f31ba1c9ac630f15dec8-Paper-Conference.pdf},
 volume = {37},
 year = {2024}
}

@misc{aime25,
  title     = {{A}merican Invitational Mathematics Examination ({AIME}) 2025},
  author    = {Zhang, Yifan and Math-AI Team},
  year      = {2025},
  publisher = {HuggingFace},
  url       = {https://huggingface.co/datasets/math-ai/aime25}
}

@inproceedings{ICLR2024_aca97732,
 author = {Lightman, Hunter and Kosaraju, Vineet and Burda, Yuri and Edwards, Harrison and Baker, Bowen and Lee, Teddy and Leike, Jan and Schulman, John  and Sutskever, Ilya and Cobbe, Karl},
 booktitle = {International Conference on Learning Representations},
 editor = {B. Kim and Y. Yue and S. Chaudhuri and K. Fragkiadaki and M. Khan and Y. Sun},
 pages = {39578--39601},
 title = {Let\textquotesingle s Verify Step by Step},
 url = {https://proceedings.iclr.cc/paper_files/paper/2024/file/aca97732e30bcf1303bc22ac3924fd16-Paper-Conference.pdf},
 volume = {2024},
 year = {2024}
}

@inproceedings{
rein2024gpqa,
title={{GPQA}: A Graduate-Level {G}oogle-Proof {Q}\&{A} Benchmark},
author={David Rein and Betty Li Hou and Asa Cooper Stickland and Jackson Petty and Richard Yuanzhe Pang and Julien Dirani and Julian Michael and Samuel R. Bowman},
booktitle={First Conference on Language Modeling},
year={2024},
url={https://openreview.net/forum?id=Ti67584b98}
}

@inproceedings{
jain2025livecodebench,
title={{L}ive{C}ode{B}ench: Holistic and Contamination Free Evaluation of Large Language Models for Code},
author={Naman Jain and King Han and Alex Gu and Wen-Ding Li and Fanjia Yan and Tianjun Zhang and Sida Wang and Armando Solar-Lezama and Koushik Sen and Ion Stoica},
booktitle={The Thirteenth International Conference on Learning Representations},
year={2025},
url={https://openreview.net/forum?id=chfJJYC3iL}
}

@inproceedings{
hsieh2024ruler,
title={{RULER}: What{\textquoteright}s the Real Context Size of Your Long-Context Language Models?},
author={Cheng-Ping Hsieh and Simeng Sun and Samuel Kriman and Shantanu Acharya and Dima Rekesh and Fei Jia and Boris Ginsburg},
booktitle={First Conference on Language Modeling},
year={2024},
url={https://openreview.net/forum?id=kIoBbc76Sy}
}

@misc{vodrahalli2024michelangelolongcontextevaluations,
      title={{M}ichelangelo: Long Context Evaluations Beyond Haystacks via Latent Structure Queries}, 
      author={Kiran Vodrahalli and Santiago Ontanon and Nilesh Tripuraneni and Kelvin Xu and Sanil Jain and Rakesh Shivanna and Jeffrey Hui and Nishanth Dikkala and Mehran Kazemi and Bahare Fatemi and Rohan Anil and Ethan Dyer and Siamak Shakeri and Roopali Vij and Harsh Mehta and Vinay Ramasesh and Quoc Le and Ed Chi and Yifeng Lu and Orhan Firat and Angeliki Lazaridou and Jean-Baptiste Lespiau and Nithya Attaluri and Kate Olszewska},
      year={2024},
      eprint={2409.12640},
      archivePrefix={arXiv},
      primaryClass={cs.CL},
      url={https://arxiv.org/abs/2409.12640}, 
}

@inproceedings{NEURIPS2024_370df50c,
 author = {Penedo, Guilherme and Kydl\'{\i}\v{c}ek, Hynek and allal, Loubna Ben and Lozhkov, Anton and Mitchell, Margaret and Raffel, Colin and Von Werra, Leandro and Wolf, Thomas},
 booktitle = {Advances in Neural Information Processing Systems},
 doi = {10.52202/079017-0970},
 editor = {A. Globerson and L. Mackey and D. Belgrave and A. Fan and U. Paquet and J. Tomczak and C. Zhang},
 pages = {30811--30849},
 publisher = {Curran Associates, Inc.},
 title = {The {F}ine{W}eb Datasets: Decanting the Web for the Finest Text Data at Scale},
 url = {https://proceedings.neurips.cc/paper_files/paper/2024/file/370df50ccfdf8bde18f8f9c2d9151bda-Paper-Datasets_and_Benchmarks_Track.pdf},
 volume = {37},
 year = {2024}
}

@inproceedings{
merity2017pointer,
title={Pointer Sentinel Mixture Models},
author={Stephen Merity and Caiming Xiong and James Bradbury and Richard Socher},
booktitle={International Conference on Learning Representations},
year={2017},
url={https://openreview.net/forum?id=Byj72udxe}
}

@InProceedings{pmlr-v267-shutova25a,
  title = 	 {Cache Me If You Must: Adaptive Key-Value Quantization for Large Language Models},
  author =       {Shutova, Alina and Malinovskii, Vladimir and Egiazarian, Vage and Kuznedelev, Denis and Mazur, Denis and Nikita, Surkov and Ermakov, Ivan and Alistarh, Dan},
  booktitle = 	 {Proceedings of the 42nd International Conference on Machine Learning},
  pages = 	 {55451--55473},
  year = 	 {2025},
  editor = 	 {Singh, Aarti and Fazel, Maryam and Hsu, Daniel and Lacoste-Julien, Simon and Berkenkamp, Felix and Maharaj, Tegan and Wagstaff, Kiri and Zhu, Jerry},
  volume = 	 {267},
  series = 	 {Proceedings of Machine Learning Research},
  month = 	 {13--19 Jul},
  publisher =    {PMLR},
  url = 	 {https://proceedings.mlr.press/v267/shutova25a.html}
}
\bibliographystyle{arxiv}
\addcontentsline{toc}{section}{References}

\appendix

\section{Alternative Key-Transform Placements}
\label{sec:alternative_key_placements}

This appendix expands the design tradeoff summarized in \cref{sec:method}.
We first characterize transforms that can pass through RoPE and then consider unrestricted transforms placed before RoPE.

\subsection{RoPE-commuting transforms}

On the basis that groups the rotary coordinate pairs, standard RoPE has the form
\begin{equation}
\bm{R}_{s} = \operatorname{blockdiag}_{j} \left( \bm{R} \left( s \omega_{j} \right) \right)
, \qquad
\bm{R} \left( \theta \right) =
\begin{pmatrix}
\cos \theta & -\sin \theta \\
\sin \theta & \cos \theta
\end{pmatrix}
.
\end{equation}
For distinct nondegenerate rotary frequencies, the real invertible transforms that commute with every $\bm{R}_{s}$ are
\begin{equation}
\bm{T}_{\mathrm{K} \left( h \right)} = \operatorname{blockdiag}_{j} \left( \rho_{h,j} \bm{R} \left( \theta_{h,j} \right) \right)
, \qquad
\bm{T}_{\mathrm{K} \left( h \right)} \bm{R}_{s} = \bm{R}_{s} \bm{T}_{\mathrm{K} \left( h \right)}
, \qquad
\rho_{h,j} > 0
.
\end{equation}
Each two-dimensional block lies in the span of $\left( \begin{smallmatrix} 1 & 0 \\ 0 & 1 \end{smallmatrix} \right)$ and $\left( \begin{smallmatrix} 0 & -1 \\ 1 & 0 \end{smallmatrix} \right)$, while the distinct frequencies prevent mixing between rotary pairs.
In the absence of post-projection normalization, such transforms and their inverse-transposes could pass through RoPE and be folded into the key and query projections.

A learned coordinatewise normalization further restricts exact folding.
For example, write RMS normalization with gain $\bm{\nu}$ as
\begin{equation}
\operatorname{Norm} \left( \bm{x} \right)
= d^{\frac{1}{2}} \left\| \bm{x} \right\|^{-1} \operatorname{diag} \left( \bm{\nu} \right) \bm{x}
.
\end{equation}
Moving a linear transform through this fixed normalization for every input requires
\begin{equation}
\bm{T} \operatorname{Norm} \left( \bm{x} \right) = \operatorname{Norm} \left( \bm{T} \bm{x} \right) \quad \forall\, \bm{x} \ne \bm{0}
\quad \Longrightarrow \quad
\bm{T}^{\top} \bm{T} = \mathbf{I}
, \qquad
\bm{T} \operatorname{diag} \left( \bm{\nu} \right) = \operatorname{diag} \left( \bm{\nu} \right) \bm{T}
.
\end{equation}
For nonzero gain entries, equality for every nonzero input forces $\bm{T}$ to commute with $\operatorname{diag} \left( \bm{\nu} \right)$ and preserve the RMS denominator ($\bm{T}^{\top} \bm{T} = \mathbf{I}$).
The corresponding conditions must hold for the key transform and its inverse-transpose under the respective key and query normalizations.
For generic learned gains, combining these conditions with RoPE commutation leaves only $\bm{T}_{\mathrm{K} \left( h \right),j} \in \left\{ \left( \begin{smallmatrix} 1 & 0 \\ 0 & 1 \end{smallmatrix} \right), \left( \begin{smallmatrix} -1 & 0 \\ 0 & -1 \end{smallmatrix} \right) \right\}$.
These paired sign flips leave coordinate magnitudes unchanged and therefore do not redistribute quantization difficulty for a symmetric quantizer such as \cref{eq:quantizer}.

If the learned normalization gains $\bm{\nu}$ are also allowed to change as part of the KV-cache quantization reparameterization, additional transforms may become foldable.
However, such reparameterizations remain constrained by RoPE and by the structure of the normalization, and introduce additional design choices beyond folding a transform into fixed projection and normalization parameters.
We therefore restrict our discussion to the fixed-normalization setting above.

\subsection{Full pre-RoPE transforms}

A generic transform placed before RoPE restores the full $d \times d$ freedom but introduces position-dependent compensation.
Let $\bm{k}_{h,s}$ and $\bm{q}_{h,g,t}$ denote the key at token position $s$ and the query at token position $t$, respectively.
Since $\bm{R}_{s}$ is orthogonal, the normalized pre-RoPE key is $\bm{R}_{s}^{\top} \bm{k}_{h,s}$.
If the cache instead stores $\bm{T}_{\mathrm{K} \left( h \right)} \bm{R}_{s}^{\top} \bm{k}_{h,s}$, then
\begin{equation}
\bm{k}_{h,s}^{\top} \bm{q}_{h,g,t}
= \left( \bm{T}_{\mathrm{K} \left( h \right)} \bm{R}_{s}^{\top} \bm{k}_{h,s} \right)^{\top} \bm{T}_{\mathrm{K} \left( h \right)}^{-\top} \bm{R}_{s}^{\top} \bm{q}_{h,g,t}
.
\end{equation}
Decoding must therefore restore each cached key before applying its positional rotation or incorporate a different compensation for every cache position.
The current post-RoPE placement instead retains a full transform and a simple cache representation at the cost of fixed online transforms for each new key and query.

\section{Clipping-Aware Near-Optimality for the QuEST Quantizer}
\label{sec:optimality}

We prove \cref{thm:near_optimal} by starting from the unclipped rounding loss, which WUSH minimizes, and then controlling the effect of clipping.
Clipping adds error outside the range and removes rounding noise from those coordinates.
Both effects must be included in the comparison.

\subsection{Rounding and clipping errors}
\label{sec:noise_model}

We keep $\bm{X}$ and $\bm{H}$ fixed, with $\bm{M} = \bm{X} \bm{X}^{\top}$ as in \cref{sec:wush}.
Each column $\bm{x}_{j}$ is one quantization group.
Following \citet{chen2026wush}, we approximate rounding with independent uniform noise, but keep clipping exact.
This is a model of the rounding error, not an identity for the deterministic quantizer.

Fix an invertible $\bm{T}$.
For group $j$, its RMS after the transform and its quantization step are
\begin{equation}
 r_{j} = d^{-\frac{1}{2}} \left\| \bm{T} \bm{x}_{j} \right\|
 , \qquad
 \Delta_{j} = 2 \alpha_{b} \left(2^{b}-1\right)^{-1} r_j
 .
\label{eq:rms_coordinates}
\end{equation}
The clipping endpoints are $\pm\alpha_{b}r_{j}$.
The error caused by clipping alone, before undoing the transform, is
\begin{equation}
\bm{e}_{j}
= \operatorname{clip} \left( \bm{T}\bm{x}_{j}, -\alpha_{b}r_{j}, \alpha_{b}r_{j} \right)
- \bm{T}\bm{x}_{j}
.
\label{eq:clipping_residual}
\end{equation}
Here $\operatorname{clip}$ acts coordinatewise, and $e_{j,i}$ denotes coordinate $i$ of group $j$.
Inside the range, $e_{j,i}=0$ and we replace the rounding error with $\Delta_{j}\xi_{j,i}$, where all $\xi_{j,i}$ are independent and uniformly distributed on $[-2^{-1},2^{-1}]$.
Outside the range, the error is exactly $e_{j,i}$.
Thus, the total round-trip perturbation $\bm{\varepsilon}_{j}$ satisfies
\begin{equation}
\left( \bm{T}\bm{\varepsilon}_{j} \right)_{i}
= \begin{cases}
\Delta_{j}\xi_{j,i}, & \left| (\bm{T}\bm{x}_{j})_{i} \right| \le \alpha_{b}r_{j}, \\
e_{j,i}, & \left| (\bm{T}\bm{x}_{j})_{i} \right| > \alpha_{b}r_{j}.
\end{cases}
\label{eq:saturated_noise_model}
\end{equation}
Zero groups have $r_j=\Delta_j=0$ and zero error.
All group quantities depend on $\bm{T}$.
We suppress that dependence in their subscripts.

To measure errors in transformed coordinates, write
$\bm{B}_{\bm{T}} = \bm{T}^{-\top}\bm{H}\bm{T}^{-1}$.
Then $\bm{\varepsilon}_{j}^{\top}\bm{H}\bm{\varepsilon}_{j}
= (\bm{T}\bm{\varepsilon}_{j})^{\top}\bm{B}_{\bm{T}}(\bm{T}\bm{\varepsilon}_{j})$.
The diagonal entry $(\bm{B}_{\bm{T}})_{ii}$ therefore measures sensitivity to an error in coordinate $i$.
The modeled output loss is $\ell(\bm{T})=2^{-1}\sum_j\bm{\varepsilon}_{j}^{\top}\bm{H}\bm{\varepsilon}_{j}$.
Expectations below are over the rounding noise.

\textbf{An unclipped reference.}
Uniform rounding noise with step $\Delta_j$ has variance $12^{-1}\Delta_j^2$.
Applying it to every coordinate without clipping would give the expected loss
\begin{equation}
\begin{aligned}
\ell_{\mathrm{rnd}}(\bm{T})
&= 24^{-1}\sum_j\Delta_j^2\operatorname{tr}(\bm{B}_{\bm{T}}) \\
&= 6^{-1}d^{-1}\alpha_b^2(2^b-1)^{-2}
\operatorname{tr}(\bm{T}\bm{M}\bm{T}^{\top})
\operatorname{tr}(\bm{B}_{\bm{T}})
.
\end{aligned}
\label{eq:expected_loss}
\end{equation}
The last equality uses $\sum_j r_j^2=d^{-1}\operatorname{tr}(\bm{T}\bm{M}\bm{T}^{\top})$.
This reference loss is a deterministic quantity, distinct from the clipped loss.

\textbf{The clipped loss.}
Independence and zero mean remove all cross terms involving rounding noise, giving
\begin{equation}
\begin{aligned}
\mathbb{E}[\ell(\bm{T})]
={}& 24^{-1}\sum_j\Delta_j^2
\sum_{i:\,|(\bm{T}\bm{x}_j)_i|\le\alpha_b r_j}(\bm{B}_{\bm{T}})_{ii} \\
&+ 2^{-1}\sum_j\bm{e}_j^{\top}\bm{B}_{\bm{T}}\bm{e}_j
.
\end{aligned}
\label{eq:clipped_loss_decomposition}
\end{equation}
The first term counts rounding only inside the clipping range.
The second measures the exact clipping error.
Clipped coordinates no longer receive rounding noise, so the total can be smaller than the unclipped reference.

\subsection{The unclipped optimum}
\label{sec:optimal_transform}

Only the trace product in \cref{eq:expected_loss} depends on the transform.
The following result shows that WUSH minimizes it over all invertible transforms.

\begin{theorem}
\label{thm:optimal}
Let $\bm{M},\bm{H}\in\mathbb{R}^{d\times d}$ be symmetric positive definite, and let $\sigma_1,\ldots,\sigma_d$ be the singular values of $\bm{M}^{\frac{1}{2}}\bm{H}^{\frac{1}{2}}$.\footnote{For a symmetric positive definite matrix, the square root is obtained by taking the positive square roots of its eigenvalues and keeping its eigenvectors.}
Every invertible $\bm{T}$ satisfies
\begin{equation}
\operatorname{tr}(\bm{T}\bm{M}\bm{T}^{\top})
\operatorname{tr}(\bm{T}^{-\top}\bm{H}\bm{T}^{-1})
\ge \left(\sum_i\sigma_i\right)^2
,
\label{eq:optimality}
\end{equation}
and \cref{eq:wush} with $\gamma=0$ attains the bound.
\end{theorem}

\begin{proof}
Take a singular value decomposition
$\bm{M}^{\frac{1}{2}}\bm{H}^{\frac{1}{2}}=\sum_i\sigma_i\bm{u}_i\bm{v}_i^{\top}$,
where $\{\bm{u}_i\}$ and $\{\bm{v}_i\}$ are orthonormal bases.
Inserting $\bm{T}^{\top}\bm{T}^{-\top}=\mathbf{I}$ and applying Cauchy-Schwarz gives
\begin{equation}
\begin{aligned}
\sum_i\sigma_i
&=\sum_i
(\bm{T}\bm{M}^{\frac{1}{2}}\bm{u}_i)^{\top}
(\bm{T}^{-\top}\bm{H}^{\frac{1}{2}}\bm{v}_i) \\
&\le
\left(\sum_i\|\bm{T}\bm{M}^{\frac{1}{2}}\bm{u}_i\|^2\right)^{\frac{1}{2}}
\left(\sum_i\|\bm{T}^{-\top}\bm{H}^{\frac{1}{2}}\bm{v}_i\|^2\right)^{\frac{1}{2}}
.
\end{aligned}
\end{equation}
The two sums of squared norms are the traces in \cref{eq:optimality}, because the singular-vector bases are orthonormal.
Squaring proves the lower bound.

For attainment, set $\gamma=0$ in \cref{eq:wush}, so that $\bm{L}\bm{L}^{\top}=\bm{H}$ and $\bm{U}\bm{\Lambda}\bm{U}^{\top}=\bm{L}^{\top}\bm{M}\bm{L}$.
Substituting the WUSH transform gives
\begin{equation}
\bm{T}\bm{M}\bm{T}^{\top}
=c^2\bm{\mathcal{H}}\bm{\Lambda}^{\frac{1}{2}}\bm{\mathcal{H}}^{\top}
,\qquad
\bm{T}^{-\top}\bm{H}\bm{T}^{-1}
=c^{-2}\bm{\mathcal{H}}\bm{\Lambda}^{\frac{1}{2}}\bm{\mathcal{H}}^{\top}
.
\label{eq:equalization}
\end{equation}
The diagonal of $\bm{\Lambda}^{\frac{1}{2}}$ contains the singular values of $\bm{M}^{\frac{1}{2}}\bm{L}$.
These are the $\sigma_i$, since $\bm{M}^{\frac{1}{2}}\bm{L}$ and $\bm{M}^{\frac{1}{2}}\bm{H}^{\frac{1}{2}}$ have the same product with their respective transposes.
The trace product is therefore $(\operatorname{tr}(\bm{\Lambda}^{\frac{1}{2}}))^2=(\sum_i\sigma_i)^2$.
\end{proof}

An orthogonal transform leaves both traces unchanged and gives $\operatorname{tr}(\bm{M})\operatorname{tr}(\bm{H})$.
This equals the minimum only when $\bm{M}$ and $\bm{H}$ are proportional.

For the clipped comparison, a transform is sensitivity-balanced when
\begin{equation}
(\bm{B}_{\bm{T}})_{ii}
=d^{-1}\operatorname{tr}(\bm{B}_{\bm{T}})
\qquad i=1,\ldots,d
.
\label{eq:sensitivity_balance}
\end{equation}
For the undamped WUSH transform this follows from \cref{eq:equalization}, since every entry of $\bm{\mathcal{H}}$ has magnitude $d^{-\frac{1}{2}}$.
Equal diagonal entries do not mean that $\bm{B}_{\bm{T}}$ is diagonal.

\subsection{The effect of clipping}
\label{sec:clipping}

We first bound clipping error for a scalar Gaussian.
Two assumptions then transfer this bound to the WUSH-transformed groups.
Finally, we bound how much clipping can reduce the rounding contribution of a competing transform.

\textbf{Gaussian clipping error.}
Let $\zeta\sim\mathcal{N}(0,1)$.
Clipping $\zeta$ to $[-\alpha,\alpha]$ incurs squared error $(|\zeta|-\alpha)_+^2$, where $(\cdot)_+$ takes the positive part.
The next lemma concerns clipping alone, not the full quantization error.

\begin{lemma}
\label{lem:gaussian_clip_bound}
For the Gaussian-optimal clip constant $\alpha_b$ in \cref{eq:quantizer},
\begin{equation}
\mathbb{E}\left[(|\zeta|-\alpha_b)_+^2\right]
\le (2^b-1)^{-2}
.
\label{eq:clip_fraction}
\end{equation}
Moreover, $\alpha_b\to\infty$ as $b\to\infty$.
\end{lemma}

\begin{proof}
For $\alpha>0$, let $\mathcal{Q}_{\alpha}$ round to the nearest of the $2^b$ equally spaced levels from $-\alpha$ to $\alpha$.
This scalar grid is fixed across samples, without a sample-dependent RMS scale.
Write $\widehat{\zeta}=\mathcal{Q}_{\alpha_b}(\zeta)$ and $\Delta=2\alpha_b(2^b-1)^{-1}$ for the reconstruction and spacing at the optimal endpoint.
All expectations below are over $\zeta$.
We first use optimality to bound $\mathbb{E}[\zeta\widehat{\zeta}]$.
We then relate this expectation to a weighted squared error that bounds the clipping error.

\textit{1. Bound the expectation using the optimal scale.}
If we change the endpoint from $\alpha_b$ to $\alpha$, each grid level is multiplied by $\alpha\alpha_b^{-1}$.
Thus $\alpha\alpha_b^{-1}\widehat{\zeta}$ is a valid level on the new grid, although it need not be the nearest one to $\zeta$.
For every $\alpha>0$,
\begin{equation}
\mathbb{E}\left[(\widehat{\zeta}-\zeta)^2\right]
\le \mathbb{E}\left[(\mathcal{Q}_{\alpha}(\zeta)-\zeta)^2\right]
\le \mathbb{E}\left[(\alpha\alpha_b^{-1}\widehat{\zeta}-\zeta)^2\right]
.
\label{eq:gaussian_rescaling_comparison}
\end{equation}
The first inequality uses the optimality of $\alpha_b$ among all endpoints.
The second uses the fact that choosing the nearest level cannot be worse than choosing the rescaled old level.
Both inequalities are equalities at $\alpha=\alpha_b$.
Consequently, the last expectation is minimized there.
Keeping $\widehat{\zeta}$ fixed and expanding that expectation gives the quadratic
\begin{equation}
\mathbb{E}\left[(\alpha\alpha_b^{-1}\widehat{\zeta}-\zeta)^2\right]
=\alpha^2\alpha_b^{-2}\mathbb{E}[\widehat{\zeta}^2]
-2\alpha\alpha_b^{-1}\mathbb{E}[\zeta\widehat{\zeta}]+1
.
\label{eq:gaussian_rescaling_quadratic}
\end{equation}
Here the last term is $\mathbb{E}[\zeta^2]=1$.
The derivative with respect to $\alpha$ vanishes at the positive minimizer $\alpha_b$, so
\begin{equation}
0=2\alpha_b^{-1}\left(\mathbb{E}[\widehat{\zeta}^2]-\mathbb{E}[\zeta\widehat{\zeta}]\right)
\qquad\Longrightarrow\qquad
\mathbb{E}[\widehat{\zeta}^2]=\mathbb{E}[\zeta\widehat{\zeta}]
.
\label{eq:gaussian_optimality_identity}
\end{equation}
This equality is what makes Cauchy-Schwarz useful.
Both second moments are finite, since $\zeta$ is standard Gaussian and $|\widehat{\zeta}|\le\alpha_b$.
Cauchy-Schwarz does not require the two variables to be independent and gives
\begin{equation}
\left|\mathbb{E}[\zeta\widehat{\zeta}]\right|^2
\le \mathbb{E}[\zeta^2]\,\mathbb{E}[\widehat{\zeta}^2]
=\mathbb{E}[\widehat{\zeta}^2]
.
\label{eq:gaussian_cauchy_schwarz}
\end{equation}
Substituting \cref{eq:gaussian_optimality_identity} into the left-hand side yields
$(\mathbb{E}[\widehat{\zeta}^2])^2\le\mathbb{E}[\widehat{\zeta}^2]$.
This second moment is nonnegative.
If it is zero, it is already at most one.
Otherwise, dividing by it gives $\mathbb{E}[\widehat{\zeta}^2]\le1$.
Using \cref{eq:gaussian_optimality_identity} once more, we obtain
\begin{equation}
0\le\mathbb{E}[\zeta\widehat{\zeta}]
=\mathbb{E}[\widehat{\zeta}^{2}]
\le1
.
\label{eq:gaussian_stationarity}
\end{equation}
Thus, the upper bound of one follows from optimality and Cauchy-Schwarz together, not from Cauchy-Schwarz alone.

\textit{2. Relate this expectation to a squared error.}
The moment bound in \cref{eq:gaussian_stationarity} does not yet bound clipping.
For that, we use the Gaussian density $\phi$, whose derivative satisfies $\phi'(x)=-x\phi(x)$.
Consider one quantization bin with reconstruction level $a$.
On this bin, $a$ is constant, and the quantization error is $a-x$.
Define
\begin{equation}
F_a(x)=a\left(8^{-1}\Delta^2-2^{-1}(x-a)^2\right)
,\qquad
F_a'(x)=a(a-x)
.
\label{eq:gaussian_bin_antiderivative}
\end{equation}
The constant $8^{-1}\Delta^2$ is chosen to remove the boundary terms in integration by parts.
Indeed, every finite endpoint of the bin is a midpoint between adjacent levels, where $|x-a|=2^{-1}\Delta$ and hence $F_a(x)=0$.
For the two outer bins, $F_a(x)\phi(x)$ also tends to zero at the infinite endpoint because the Gaussian density decays faster than this quadratic grows.
It follows that
\begin{equation}
\begin{aligned}
\int_{\mathrm{bin}}a(a-x)\phi(x)\,\mathrm{d}x
&=\left[F_a(x)\phi(x)\right]_{\mathrm{endpoints}}
-\int_{\mathrm{bin}}F_a(x)\phi'(x)\,\mathrm{d}x \\
&=\int_{\mathrm{bin}}xF_a(x)\phi(x)\,\mathrm{d}x \\
&=8^{-1}\Delta^2\int_{\mathrm{bin}}xa\phi(x)\,\mathrm{d}x
-2^{-1}\int_{\mathrm{bin}}xa(x-a)^2\phi(x)\,\mathrm{d}x
.
\end{aligned}
\label{eq:gaussian_bin_integration}
\end{equation}
Now sum this equality over all bins.
On the bin with level $a$, the reconstruction $\widehat{\zeta}$ equals $a$.
The three integrals therefore become
\begin{equation}
\mathbb{E}[\widehat{\zeta}(\widehat{\zeta}-\zeta)]
=8^{-1}\Delta^2\mathbb{E}[\zeta\widehat{\zeta}]
-2^{-1}\mathbb{E}\left[\zeta\widehat{\zeta}(\widehat{\zeta}-\zeta)^2\right]
.
\label{eq:gaussian_summed_bins}
\end{equation}
The left-hand side is zero by \cref{eq:gaussian_optimality_identity}.
Moving the last term to the other side and multiplying by two gives the equality below.
The inequality then follows by substituting the upper bound $\mathbb{E}[\zeta\widehat{\zeta}]\le1$ from \cref{eq:gaussian_stationarity}.
\begin{equation}
\mathbb{E}\left[\zeta\widehat{\zeta}(\widehat{\zeta}-\zeta)^2\right]
=4^{-1}\Delta^2\mathbb{E}[\zeta\widehat{\zeta}]
\le4^{-1}\Delta^2
.
\label{eq:gaussian_bin_identity}
\end{equation}
This is a bound on the total squared quantization error weighted by $\zeta\widehat{\zeta}$.
The weight lets us extract the clipping error in the next step.

\textit{3. Keep the contribution from clipped samples.}
The symmetric grid reconstructs each nonzero input with the same sign, so $\zeta\widehat{\zeta}\ge0$ everywhere.
When $|\zeta|>\alpha_b$, the reconstruction is the endpoint with that sign.
Thus $\zeta\widehat{\zeta}=\alpha_b|\zeta|\ge\alpha_b^2$ and $(\widehat{\zeta}-\zeta)^2=(|\zeta|-\alpha_b)^2$.
When $|\zeta|\le\alpha_b$, the clipping error is zero and the weighted squared error is nonnegative.
Combining these two cases gives the pointwise inequality
\begin{equation}
\alpha_b^2(|\zeta|-\alpha_b)_+^2
\le\zeta\widehat{\zeta}(\widehat{\zeta}-\zeta)^2
.
\label{eq:gaussian_pointwise_clipping}
\end{equation}
Taking expectations, applying \cref{eq:gaussian_bin_identity}, and then multiplying by $\alpha_b^{-2}$ yields
\begin{equation}
\mathbb{E}\left[(|\zeta|-\alpha_b)_+^2\right]
\le\alpha_b^{-2}\mathbb{E}\left[\zeta\widehat{\zeta}(\widehat{\zeta}-\zeta)^2\right]
\le4^{-1}\alpha_b^{-2}\Delta^2
=(2^b-1)^{-2}
.
\label{eq:gaussian_clipping_conclusion}
\end{equation}
The last equality uses $\Delta=2\alpha_b(2^b-1)^{-1}$ and proves \cref{eq:clip_fraction}.

\textit{4. Show that the optimal endpoint grows with bitwidth.}
Consider the candidate endpoint $\alpha=\sqrt{b}$.
Inside its clipping range, the error is at most half the grid spacing, or $\sqrt{b}(2^b-1)^{-1}$.
Outside the range, the error is exactly the distance to the nearest endpoint.
Optimality of $\alpha_b$ therefore gives
\begin{equation}
\mathbb{E}\left[(\widehat{\zeta}-\zeta)^2\right]
\le\mathbb{E}\left[(\mathcal{Q}_{\sqrt{b}}(\zeta)-\zeta)^2\right]
\le b(2^b-1)^{-2}
+\mathbb{E}\left[(|\zeta|-\sqrt{b})_+^2\right]
\longrightarrow0
.
\label{eq:gaussian_vanishing_error}
\end{equation}
The first term tends to zero because the number of levels grows exponentially in $b$.
The second tends to zero by dominated convergence since its integrand tends pointwise to zero and is bounded by the integrable random variable $\zeta^2$.
If $\alpha_b$ stayed bounded along a subsequence, the Gaussian tail beyond a fixed finite endpoint would give a positive lower bound on the clipping error along that subsequence.
The total quantization error is at least its clipping contribution, contradicting \cref{eq:gaussian_vanishing_error}.
Hence $\alpha_b\to\infty$.
\end{proof}

\textbf{Conditions on the transformed groups.}
For the next two conditions, evaluate $r_j$ and $\bm{e}_j$ at $\bm{T}_{\star}$ and write $\bm{B}_{\star}=\bm{B}_{\bm{T}_{\star}}$.
Dividing each nonzero transformed group by its own RMS gives
\begin{equation}
\widetilde{\bm{x}}_j=r_j^{-1}\bm{T}_{\star}\bm{x}_j,
\qquad \|\widetilde{\bm{x}}_j\|^2=d.
\label{eq:normalized_group_energy}
\end{equation}
Set $\widetilde{\bm{x}}_j=\bm{0}$ for a zero group.
Let $C_{\mathrm{tail}},C_{\mathrm{align}}\ge1$ be fixed bounds for the following conditions.

The first condition bounds the fraction of normalized entries above each threshold by a multiple of the Gaussian tail,
\begin{equation}
\left(d\sum_j r_j^2\right)^{-1}
\sum_j r_j^2\sum_{i=1}^{d}\mathbbm{1}\{|\widetilde{x}_{j,i}|>\alpha\}
\le C_{\mathrm{tail}}\,\mathbb{P}(|\zeta|>\alpha)
\qquad \forall\,\alpha\ge\alpha_b
.
\label{eq:normalized_tail}
\end{equation}
The weights $r_j^2$ account for each group's contribution to squared error.
The second condition bounds how strongly clipping errors align with sensitive directions,
\begin{equation}
\sum_j\bm{e}_j^{\top}\bm{B}_{\star}\bm{e}_j
\le C_{\mathrm{align}}d^{-1}\operatorname{tr}(\bm{B}_{\star})
\sum_j\|\bm{e}_j\|^2
.
\label{eq:clipping_alignment}
\end{equation}
This controls the direction of the clipping error, not its magnitude.
It does not require independent errors and does not follow from equal diagonal entries alone.
Both conditions concern WUSH only.
A competing transform needs only to be sensitivity-balanced.
The bounds are independent of $d$, and must also remain fixed across bitwidths for the asymptotic statement in \cref{thm:near_optimal}.

\textbf{Upper bound at WUSH.}
Integrating \cref{eq:normalized_tail} against $2(\alpha-\alpha_b)\,\mathrm{d}\alpha$ over $\alpha\ge\alpha_b$ converts the tail bound into a bound on squared clipping error.
Then \cref{eq:clip_fraction} gives
\begin{equation}
\begin{aligned}
\sum_j\|\bm{e}_j\|^2
&=\sum_j r_j^2\sum_i(|\widetilde{x}_{j,i}|-\alpha_b)_+^2 \\
&\le C_{\mathrm{tail}}\,d\sum_j r_j^2\,
\mathbb{E}\left[(|\zeta|-\alpha_b)_+^2\right] \\
&\le C_{\mathrm{tail}}\,d(2^b-1)^{-2}\sum_j r_j^2
.
\end{aligned}
\label{eq:pooled_clipping_bound}
\end{equation}
By \cref{eq:clipping_alignment}, the clipping contribution to the loss is therefore at most
\begin{equation}
\begin{aligned}
2^{-1}\sum_j\bm{e}_j^{\top}\bm{B}_{\star}\bm{e}_j
&\le 2^{-1}C_{\mathrm{tail}}C_{\mathrm{align}}(2^b-1)^{-2}
\operatorname{tr}(\bm{B}_{\star})\sum_j r_j^2 \\
&=3C_{\mathrm{tail}}C_{\mathrm{align}}\alpha_b^{-2}
\ell_{\mathrm{rnd}}(\bm{T}_{\star})
.
\end{aligned}
\label{eq:weighted_clipping_contribution}
\end{equation}
The rounding contribution cannot exceed the unclipped reference, so \cref{eq:clipped_loss_decomposition} yields
\begin{equation}
\mathbb{E}[\ell(\bm{T}_{\star})]
\le\left(1+3C_{\mathrm{tail}}C_{\mathrm{align}}\alpha_b^{-2}\right)
\ell_{\mathrm{rnd}}(\bm{T}_{\star})
.
\label{eq:expected_loss_clipped}
\end{equation}

\textbf{Lower bound for a competing transform.}
Now take any sensitivity-balanced invertible $\bm{T}$.
For each nonzero group, $\|\bm{T}\bm{x}_j\|^2=d r_j^2$, so at most $d\alpha_b^{-2}$ coordinates can exceed $\alpha_b r_j$ in magnitude.
At least a fraction $1-\alpha_b^{-2}$ therefore retain their rounding noise.
All coordinates have the same sensitivity by \cref{eq:sensitivity_balance}, so they retain at least that fraction of the reference loss.
The clipping contribution is nonnegative, giving
\begin{equation}
\mathbb{E}[\ell(\bm{T})]
\ge\left(1-\alpha_b^{-2}\right)\ell_{\mathrm{rnd}}(\bm{T})
.
\label{eq:balanced_loss_lower_bound}
\end{equation}
Zero groups contribute no loss to either side.
Finally, \cref{thm:optimal} gives $\ell_{\mathrm{rnd}}(\bm{T}_{\star})\le\ell_{\mathrm{rnd}}(\bm{T})$.
Combining the upper and lower bounds proves, for $\alpha_b>1$,
\begin{equation}
\mathbb{E}[\ell(\bm{T}_{\star})]
\le\left(1+(1+3C_{\mathrm{tail}}C_{\mathrm{align}})(\alpha_b^2-1)^{-1}\right)
\mathbb{E}[\ell(\bm{T})]
.
\label{eq:near_optimal_explicit}
\end{equation}
For fixed assumption bounds, this factor is independent of the group width.
Since $\alpha_b\to\infty$, it is $1+O(\alpha_b^{-2})$, proving \cref{thm:near_optimal}.
The comparison remains restricted to sensitivity-balanced transforms because balancing an arbitrary transform can change its clipped loss.

\section{Exact Hessians for Grouped-Query Attention}
\label{sec:exact_gqa_hessians}

This appendix derives attention-aware Hessians for perturbations of the queries, keys, and values in grouped-query attention.
Unlike the per-head Hessians in \cref{eq:wushkv}, these exact expressions depend on the position of the perturbed vector.
The final subsection aggregates the key and value Hessians to construct WUSH-A.
The query case is included for completeness.

\subsection{Setup}

Fix a key/value head $h$, and let $g = 1, \dots, n_{\mathrm{q}} / n_{\mathrm{kv}}$ index the query heads that read it.
We allow the query/key and value head dimensions to differ in this appendix, writing them as $d_{\mathrm{k}}$ and $d_{\mathrm{v}}$ (the main text takes both to be $d$).
Write
\begin{equation}
\begin{gathered}
\bm{Q}_{h,g} = \left[ \bm{q}_{h,g,1}, \dots, \bm{q}_{h,g,S} \right] \in \mathbb{R}^{d_{\mathrm{k}} \times S}
, \\
\bm{K}_{h} = \left[ \bm{k}_{h,1}, \dots, \bm{k}_{h,S} \right] \in \mathbb{R}^{d_{\mathrm{k}} \times S}
, \qquad
\bm{V}_{h} = \left[ \bm{v}_{h,1}, \dots, \bm{v}_{h,S} \right] \in \mathbb{R}^{d_{\mathrm{v}} \times S}
.
\end{gathered}
\label{eq:gqa_vectors}
\end{equation}
The query and key vectors are taken after $\operatorname{Norm}$ and $\operatorname{RoPE}$, as in \cref{eq:attention}.
Let $\bm{p}_{h,g,t}$ be column $t$ of $\bm{P}_{h,g}$, with entry $p_{h,g,t,s}$ at cache position $s$, and let $\bm{o}_{h,g,t} = \bm{V}_{h} \bm{p}_{h,g,t}$ be the corresponding head output.
The causal mask gives $p_{h,g,t,s} = 0$ whenever $s > t$.
The output-projection blocks $\bm{W}_{\mathrm{O} \left( h, g \right)} \in \mathbb{R}^{d_{\mathrm{v}} \times d_{\mathrm{model}}}$ induce the head-coupling matrices
\begin{equation}
\bm{\Gamma}_{h,g,g'} = \bm{W}_{\mathrm{O} \left( h, g \right)} \bm{W}_{\mathrm{O} \left( h, g' \right)}^{\top} \in \mathbb{R}^{d_{\mathrm{v}} \times d_{\mathrm{v}}}
.
\label{eq:gqa_coupling}
\end{equation}
The diagonal block $\bm{\Gamma}_{h,g,g}$ measures the sensitivity of query head $\left( h, g \right)$ after output projection, while $\bm{\Gamma}_{h,g,g'}$ for $g \ne g'$ couples two query heads that share $h$.
Moreover, $\bm{\Gamma}_{h,g,g'}^{\top} = \bm{\Gamma}_{h,g',g}$.

\subsection{Hessian from Attention-Output Derivatives}

To measure the local sensitivity of the attention output to quantization error, perturb one query, key, or value vector by $\bm{\delta}$.
Write $\bm{Y} = \left[ \bm{y}_{1}, \dots, \bm{y}_{S} \right]$ and $\bm{Y}' \left( \bm{\delta} \right) = \left[ \bm{y}'_{1} \left( \bm{\delta} \right), \dots, \bm{y}'_{S} \left( \bm{\delta} \right) \right]$ for the unperturbed and perturbed attention outputs, respectively, and define
\begin{equation}
\ell \left( \bm{\delta} \right) = \left\| \bm{Y}' \left( \bm{\delta} \right) - \bm{Y} \right\|_{\mathrm{F}}^{2}
.
\label{eq:gqa_loss}
\end{equation}
Let $y'_{t,j} \left( \bm{\delta} \right)$ and $y_{t,j}$ denote coordinate $j$ of $\bm{y}'_{t} \left( \bm{\delta} \right)$ and $\bm{y}_{t}$, respectively.
Differentiating \cref{eq:gqa_loss} once gives
\begin{equation}
\frac{\partial \ell}{\partial \bm{\delta}}
= 2 \sum_{t=1}^{S} \left( \frac{\partial \bm{y}'_{t}}{\partial \bm{\delta}^{\top}} \right)^{\top} \left( \bm{y}'_{t} - \bm{y}_{t} \right)
.
\label{eq:gqa_loss_gradient}
\end{equation}
Differentiating again gives the full Hessian
\begin{equation}
\frac{\partial^{2} \ell}{\partial \bm{\delta} \partial \bm{\delta}^{\top}}
= 2 \sum_{t=1}^{S} \left( \left( \frac{\partial \bm{y}'_{t}}{\partial \bm{\delta}^{\top}} \right)^{\top} \left( \frac{\partial \bm{y}'_{t}}{\partial \bm{\delta}^{\top}} \right) + \sum_{j=1}^{d_{\mathrm{model}}} \left( y'_{t,j} - y_{t,j} \right) \frac{\partial^{2} y'_{t,j}}{\partial \bm{\delta} \partial \bm{\delta}^{\top}} \right)
.
\label{eq:gqa_loss_hessian}
\end{equation}
At zero perturbation,
\begin{equation}
\bm{Y}' \left( \bm{0} \right) - \bm{Y} = \bm{0}
.
\label{eq:gqa_zero_residual}
\end{equation}
This single fact forces the gradient in \cref{eq:gqa_loss_gradient} to vanish and, independently, removes the second term in \cref{eq:gqa_loss_hessian}.
Therefore,
\begin{equation}
\begin{gathered}
\left. \frac{\partial \ell}{\partial \bm{\delta}} \right|_{\bm{\delta} = \bm{0}} = \bm{0}
, \\
\left. \frac{\partial^{2} \ell}{\partial \bm{\delta} \partial \bm{\delta}^{\top}} \right|_{\bm{\delta} = \bm{0}} = 2 \sum_{t=1}^{S} \left. \left( \frac{\partial \bm{y}'_{t}}{\partial \bm{\delta}^{\top}} \right)^{\top} \left( \frac{\partial \bm{y}'_{t}}{\partial \bm{\delta}^{\top}} \right) \right|_{\bm{\delta} = \bm{0}}
.
\end{gathered}
\label{eq:gqa_hessian_jacobian}
\end{equation}

\subsection{Query Hessian}

Perturb the query $\bm{q}_{h,g,t}$ by $\bm{\delta} \in \mathbb{R}^{d_{\mathrm{k}}}$,
\begin{equation}
\bm{q}_{h,g,t} \mapsto \bm{q}_{h,g,t} + \bm{\delta}
.
\label{eq:gqa_query_perturbation}
\end{equation}
Only the query head $\left( h, g \right)$ at position $t$ changes.
The derivative of the logit vector with respect to $\bm{\delta}^{\top}$ is $\tau^{-1} \bm{K}_{h}^{\top}$, while the softmax derivative is $\operatorname{diag} \left( \bm{p}_{h,g,t} \right) - \bm{p}_{h,g,t} \bm{p}_{h,g,t}^{\top}$.
Chaining them gives the head-output and attention-output derivatives
\begin{equation}
\begin{gathered}
\left. \frac{\partial \bm{o}'_{h,g,t}}{\partial \bm{\delta}^{\top}} \right|_{\bm{\delta} = \bm{0}} = \tau^{-1} \bm{V}_{h} \left( \operatorname{diag} \left( \bm{p}_{h,g,t} \right) - \bm{p}_{h,g,t} \bm{p}_{h,g,t}^{\top} \right) \bm{K}_{h}^{\top}
, \\
\left. \frac{\partial \bm{y}'_{t}}{\partial \bm{\delta}^{\top}} \right|_{\bm{\delta} = \bm{0}} = \tau^{-1} \bm{W}_{\mathrm{O} \left( h, g \right)}^{\top} \bm{V}_{h} \left( \operatorname{diag} \left( \bm{p}_{h,g,t} \right) - \bm{p}_{h,g,t} \bm{p}_{h,g,t}^{\top} \right) \bm{K}_{h}^{\top}
.
\end{gathered}
\label{eq:gqa_query_jacobian}
\end{equation}
Substituting the attention-output derivative into \cref{eq:gqa_hessian_jacobian} gives the $d_{\mathrm{k}} \times d_{\mathrm{k}}$ query Hessian
\begin{equation}
\begin{array}{@{}l@{}}
\displaystyle \bm{H}_{\mathrm{Q} \left( h, g, t \right)}
= 2 \left. \left( \frac{\partial \bm{y}'_{t}}{\partial \bm{\delta}^{\top}} \right)^{\top} \left( \frac{\partial \bm{y}'_{t}}{\partial \bm{\delta}^{\top}} \right) \right|_{\bm{\delta} = \bm{0}}
\\
\displaystyle {}= 2 \tau^{-2} \bm{K}_{h} \left( \operatorname{diag} \left( \bm{p}_{h,g,t} \right) - \bm{p}_{h,g,t} \bm{p}_{h,g,t}^{\top} \right) \bm{V}_{h}^{\top} \bm{\Gamma}_{h,g,g} \bm{V}_{h} \left( \operatorname{diag} \left( \bm{p}_{h,g,t} \right) - \bm{p}_{h,g,t} \bm{p}_{h,g,t}^{\top} \right) \bm{K}_{h}^{\top}
.
\end{array}
\label{eq:gqa_query_hessian}
\end{equation}

\subsection{Key Hessian}

Perturb the shared key $\bm{k}_{h,s}$ by $\bm{\delta} \in \mathbb{R}^{d_{\mathrm{k}}}$,
\begin{equation}
\bm{k}_{h,s} \mapsto \bm{k}_{h,s} + \bm{\delta}
.
\label{eq:gqa_key_vector_perturbation}
\end{equation}
This changes every query head that reads $h$ and every causal query position $t \ge s$.
Differentiating the softmax with respect to the logit at cache position $s$ and summing against $\bm{V}_{h}$ gives the residual $\bm{v}_{h,s} - \bm{o}_{h,g,t}$.
For one such head and position, the head-output derivative is
\begin{equation}
\left. \frac{\partial \bm{o}'_{h,g,t}}{\partial \bm{\delta}^{\top}} \right|_{\bm{\delta} = \bm{0}} = \tau^{-1} p_{h,g,t,s} \left( \bm{v}_{h,s} - \bm{o}_{h,g,t} \right) \bm{q}_{h,g,t}^{\top}
.
\label{eq:gqa_key_perturbation}
\end{equation}
The corresponding attention-output derivative is
\begin{equation}
\left. \frac{\partial \bm{y}'_{t}}{\partial \bm{\delta}^{\top}} \right|_{\bm{\delta} = \bm{0}} = \tau^{-1} \sum_{g} p_{h,g,t,s} \bm{W}_{\mathrm{O} \left( h, g \right)}^{\top} \left( \bm{v}_{h,s} - \bm{o}_{h,g,t} \right) \bm{q}_{h,g,t}^{\top}
.
\label{eq:gqa_key_jacobian}
\end{equation}
Substituting the attention-output derivative into \cref{eq:gqa_hessian_jacobian} and expanding the head interactions gives the $d_{\mathrm{k}} \times d_{\mathrm{k}}$ key Hessian
\begin{equation}
\begin{array}{@{}l@{}}
\displaystyle \bm{H}_{\mathrm{K} \left( h, s \right)}
= 2 \sum_{t=1}^{S} \left. \left( \frac{\partial \bm{y}'_{t}}{\partial \bm{\delta}^{\top}} \right)^{\top} \left( \frac{\partial \bm{y}'_{t}}{\partial \bm{\delta}^{\top}} \right) \right|_{\bm{\delta} = \bm{0}}
\\
\displaystyle {}= 2 \tau^{-2} \sum_{t=1}^{S} \sum_{g,g'} p_{h,g,t,s} p_{h,g',t,s} \left( \bm{v}_{h,s} - \bm{o}_{h,g,t} \right)^{\top} \bm{\Gamma}_{h,g,g'} \left( \bm{v}_{h,s} - \bm{o}_{h,g',t} \right) \bm{q}_{h,g,t} \bm{q}_{h,g',t}^{\top}
.
\end{array}
\label{eq:gqa_key_hessian}
\end{equation}
The residual bilinear form is scalar, so each summand is a scaled query outer product.

\subsection{Value Hessian}

Perturb the shared value $\bm{v}_{h,s}$ by $\bm{\delta} \in \mathbb{R}^{d_{\mathrm{v}}}$,
\begin{equation}
\bm{v}_{h,s} \mapsto \bm{v}_{h,s} + \bm{\delta}
.
\label{eq:gqa_value_vector_perturbation}
\end{equation}
Values do not enter the softmax, so the attention probabilities remain fixed.
For query head $\left( h, g \right)$ at position $t \ge s$, the head-output derivative is
\begin{equation}
\left. \frac{\partial \bm{o}'_{h,g,t}}{\partial \bm{\delta}^{\top}} \right|_{\bm{\delta} = \bm{0}} = p_{h,g,t,s} \mathbf{I}_{d_{\mathrm{v}}}
.
\label{eq:gqa_value_perturbation}
\end{equation}
The attention-output derivative is, therefore,
\begin{equation}
\left. \frac{\partial \bm{y}'_{t}}{\partial \bm{\delta}^{\top}} \right|_{\bm{\delta} = \bm{0}} = \sum_{g} p_{h,g,t,s} \bm{W}_{\mathrm{O} \left( h, g \right)}^{\top}
.
\label{eq:gqa_value_jacobian}
\end{equation}
Substituting into \cref{eq:gqa_hessian_jacobian} gives the $d_{\mathrm{v}} \times d_{\mathrm{v}}$ value Hessian
\begin{equation}
\begin{array}{@{}l@{}}
\displaystyle \bm{H}_{\mathrm{V} \left( h, s \right)}
= 2 \sum_{t=1}^{S} \left. \left( \frac{\partial \bm{y}'_{t}}{\partial \bm{\delta}^{\top}} \right)^{\top} \left( \frac{\partial \bm{y}'_{t}}{\partial \bm{\delta}^{\top}} \right) \right|_{\bm{\delta} = \bm{0}}
\\
\displaystyle {}= 2 \sum_{t=1}^{S} \sum_{g,g'} p_{h,g,t,s} p_{h,g',t,s} \bm{\Gamma}_{h,g,g'}
.
\end{array}
\label{eq:gqa_value_hessian}
\end{equation}

\subsection{Aggregation for WUSH-A}

For the calibration sequence $i$, let $\bm{H}_{\mathrm{K} \left( h, s \right)}^{\left( i \right)}$ and $\bm{H}_{\mathrm{V} \left( h, s \right)}^{\left( i \right)}$ denote the per-position Hessians above.
Using the Gram matrices from \cref{alg:calibration}, the resulting transforms are
\begin{equation}
\begin{gathered}
\bm{T}_{\mathrm{K} \left( h \right)} = \textsc{Wush} \left( \sum_{i = 1}^{n_{\mathrm{cal}}} \bm{M}_{\mathrm{K} \left( h \right)}^{\left( i \right)}, \sum_{i = 1}^{n_{\mathrm{cal}}} \sum_{s = 1}^{S} \bm{H}_{\mathrm{K} \left( h, s \right)}^{\left( i \right)} \right)
, \\
\bm{T}_{\mathrm{V} \left( h \right)} = \textsc{Wush} \left( \sum_{i = 1}^{n_{\mathrm{cal}}} \bm{M}_{\mathrm{V} \left( h \right)}^{\left( i \right)}, \sum_{i = 1}^{n_{\mathrm{cal}}} \sum_{s = 1}^{S} \bm{H}_{\mathrm{V} \left( h, s \right)}^{\left( i \right)} \right)
.
\end{gathered}
\label{eq:wush_a_transforms}
\end{equation}
Thus, WUSH-A differs from WUSH only in the Hessian supplied to the transform construction.

\section{Additional Experimental Results}

\subsection{Additional Attention-Stage Errors}
\label{sec:additional_attention_stage_errors}

The 2-bit diagnostic plots from \cref{sec:experiments_attention_error} appear in \cref{fig:attention_stage_diagnostics}, and the 3-bit and 4-bit module-output errors appear in \cref{fig:attention_stage_combined}.
For the diagnostics, query-key dot products are measured before scaling and masking, while softmax uses the model's native scaling and causal mask.
The single-sided module-output plots quantize only keys or only values.

\begin{figure}[!htb]
\begin{center}
\includegraphics[width=0.49\textwidth]{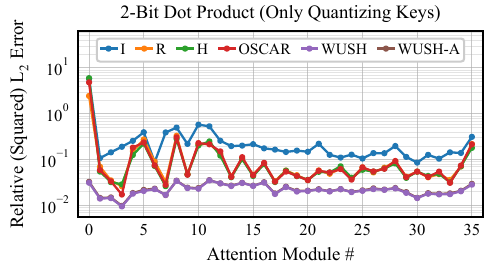}
\hfill
\includegraphics[width=0.49\textwidth]{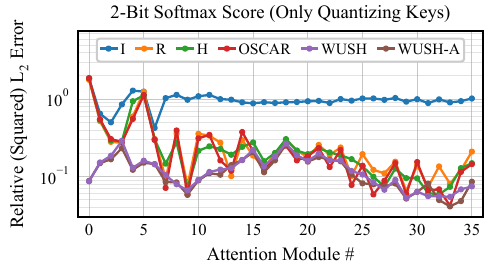}
\\[0.5em]
\includegraphics[width=0.49\textwidth]{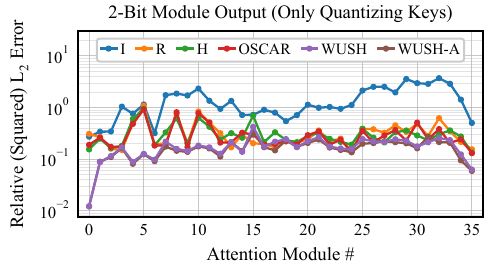}
\hfill
\includegraphics[width=0.49\textwidth]{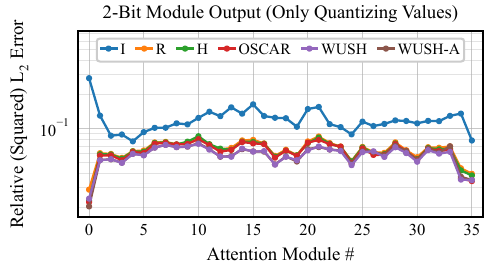}
\end{center}
\caption{Relative (squared) $\mathrm{L}_{2}$ errors at 2-bit for I, R, H, OSCAR, WUSH, and WUSH-A. The top row shows query-key dot-product and softmax errors with only keys quantized. The bottom row shows module-output errors with only keys or only values quantized.}
\label{fig:attention_stage_diagnostics}
\end{figure}

\begin{figure}[!htb]
\begin{center}
\includegraphics[width=0.49\textwidth]{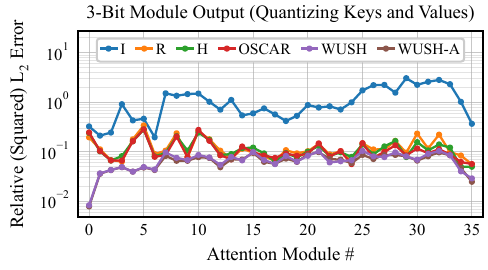}
\hfill
\includegraphics[width=0.49\textwidth]{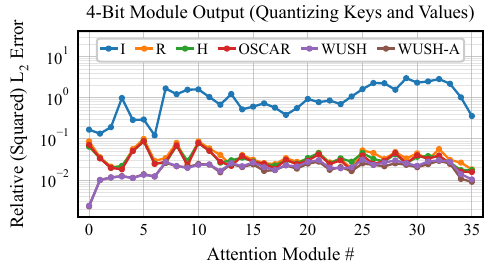}
\end{center}
\caption{Relative (squared) $\mathrm{L}_{2}$ module-output errors at 3-bit and 4-bit for I, R, H, OSCAR, WUSH, and WUSH-A.}
\label{fig:attention_stage_combined}
\end{figure}

\FloatBarrier

\subsection{Quantizer and Calibration Choices}
\label{sec:quantizer_calibration_choices}

\cref{tab:quantizer_calibration_choices} compares the main Qwen3-8B configurations with two alternatives under the same 2-bit benchmark protocol.
The main rows repeat the scores in \cref{tab:downstream_benchmarks}, and the alternative rows come from separate control evaluations.

\begin{table}[!htb]
\caption{Qwen3-8B downstream scores (\%) over three seeds for the main configurations and quantizer or calibration alternatives.
The gray second line gives the mean generated-token count in thousands (k) and the budget-exhaustion rate.
Main configurations are marked in bold.}
\label{tab:quantizer_calibration_choices}
\newcommand{\choicescore}[2]{#1\,$\pm$\,#2}
\newcommand{\choicestats}[2]{\textcolor{black!65}{#1\enspace#2\%}}
\begin{center}
\begin{tabular}{@{}l|cccc@{}}
\toprule
Calibration, Quantizer & AIME 2025 & MATH-500 & GPQA Diamond & LiveCodeBench v6 \\
\midrule
\emph{WUSH-KV}\\[0.35em]
\multirow{2}{*}{\textbf{FineWeb-Edu, OSCAR}} & \textbf{\choicescore{46.7}{3.3}} & \textbf{\choicescore{92.9}{0.1}} & \textbf{\choicescore{54.5}{2.3}} & \textbf{\choicescore{35.2}{0.7}} \\
& \choicestats{53.4k}{12} & \choicestats{12.2k}{2} & \choicestats{16.3k}{1} & \choicestats{53.9k}{26} \\[0.35em]
\multirow{2}{*}{FineWeb-Edu, QuEST} & \choicescore{28.9}{6.9} & \choicescore{88.7}{1.0} & \choicescore{50.5}{2.5} & \choicescore{33.1}{0.6} \\
& \choicestats{63.4k}{42} & \choicestats{11.7k}{4} & \choicestats{22.1k}{4} & \choicestats{63.7k}{46} \\
\midrule
\emph{OSCAR}\\[0.35em]
\multirow{2}{*}{\textbf{GPQA Diamond, OSCAR}} & \textbf{\choicescore{34.4}{8.4}} & \textbf{\choicescore{89.3}{1.1}} & \textbf{\choicescore{52.9}{3.0}} & \textbf{\choicescore{20.4}{0.3}} \\
& \choicestats{63.0k}{42} & \choicestats{16.0k}{6} & \choicestats{18.3k}{1} & \choicestats{73.8k}{71} \\[0.35em]
\multirow{2}{*}{FineWeb-Edu, OSCAR} & \choicescore{22.2}{1.9} & \choicescore{76.9}{1.3} & \choicescore{48.7}{0.3} & \choicescore{13.1}{1.2} \\
& \choicestats{70.8k}{60} & \choicestats{21.2k}{18} & \choicestats{19.2k}{4} & \choicestats{79.1k}{80} \\
\bottomrule
\end{tabular}
\end{center}
\end{table}

\textbf{Quantizer choice.}
With the WUSH-KV transform and FineWeb-Edu calibration setup held fixed, the OSCAR-style quantizer yields higher end-to-end scores than QuEST on all four tasks.
In a separate single-run AIME 2025 diagnostic, OSCAR calibrated on GPQA Diamond with QuEST quantizer scored 0.0\%.
In our local layerwise and attention-module-wise measurements, QuEST instead has lower $\mathrm{L}_{2}$ quantization error for both transforms.
Thus local $\mathrm{L}_{2}$ reconstruction error need not predict downstream autoregressive quality, motivating the OSCAR-style quantizer in the main evaluation.

\textbf{Calibration choice.}
WUSH-KV uses 128 FineWeb-Edu sequences of length 32,768, or about 4.2 million calibration tokens.
By comparison, OSCAR's Qwen3-8B calibration targets about 8k GPQA Diamond tokens.
Two preliminary AIME 2025 runs with GPQA-calibrated WUSH-KV and OSCAR-style quantizer scored 16.7\% and 10.0\%, suggesting that this setup was suboptimal for WUSH-KV in our tests.
We also recalibrated OSCAR on a FineWeb-Edu set matched in size to WUSH-KV's calibration set to test whether more calibration data would improve its results.
This alternative scores lower than the released OSCAR transforms calibrated on the GPQA Diamond on all four tasks, so we retain the released transforms in the main comparison.

\subsection{Long-Context Evaluations}
\label{sec:long_context_evaluations}

We evaluate Qwen3-8B and Qwen3-32B on long-context tasks with 2-bit KV caches and BF16 baselines using the SGLang setup of \cref{sec:experiments_benchmarks}.
For Qwen3-8B, \cref{tab:ruler_niah_length} reports RULER NIAH results from 4k to 128k, and \cref{tab:mrcr_summary,tab:mrcr_by_length} reports OpenAI MRCR results across 0k-128k.
We report the normalized area under the MRCR score vs context-length curve (AUC), following Context Arena.
The separate Qwen3-32B MRCR results appear in \cref{tab:mrcr_32b}.
Scores in this subsection are averaged over multiple seeds, and $\pm$ denotes standard deviation when shown.

In the RULER NIAH results, WUSH-KV remains close to BF16 at 4k and 8k but degrades as the context length increases.
WUSH-KV scores above OSCAR at every tested length, with the largest gap at 128k.

MRCR is difficult for all quantized methods, with substantial losses relative to BF16.
On Qwen3-8B, WUSH-KV generally degrades more gradually across context lengths and retains higher scores at long contexts than the OSCAR baseline.
On Qwen3-32B, OSCAR has slightly higher aggregate scores, while WUSH-KV scores higher at 64k-128k for all three needle counts.

\begin{table}[!htb]
\caption{Qwen3-8B RULER NIAH accuracy (\%) over eight subtasks by context length.}
\label{tab:ruler_niah_length}
\begin{center}
\begin{tabular}{@{}l|ccccccc@{}}
\toprule
Method & 4k & 8k & 16k & 32k & 64k & 128k & Overall \\
\midrule
BF16 & 99.71 & 97.79 & 92.85 & 94.29 & 85.09 & 79.17 & 91.48 \\
WUSH-KV & \textbf{97.44} & \textbf{91.92} & \textbf{80.75} & \textbf{76.87} & \textbf{68.36} & \textbf{57.26} & \textbf{78.77} \\
OSCAR & \textbf{96.63} & \textbf{89.79} & \textbf{80.00} & 74.00 & 64.36 & 25.69 & 71.74 \\
\bottomrule
\end{tabular}
\end{center}
\end{table}

\newcommand{\longcontextscore}[2]{\shortstack{#1\,$\pm$\,#2}}

\begin{table}[!htb]
\caption{Qwen3-8B MRCR summary accuracy (\%).}
\label{tab:mrcr_summary}
\begin{center}
\begin{tabular}{@{}l|ccccc@{}}
\toprule
Method & 2-Needle & 4-Needle & 8-Needle & Overall & AUC \\
\midrule
BF16 & \longcontextscore{36.89}{1.96} & \longcontextscore{22.66}{1.82} & \longcontextscore{16.38}{0.44} & 25.31 & 21.95 \\
WUSH-KV & \textbf{\longcontextscore{10.55}{0.24}} & \textbf{\longcontextscore{9.66}{0.37}} & \textbf{\longcontextscore{8.45}{0.24}} & \textbf{9.55} & \textbf{8.35} \\
OSCAR & \longcontextscore{7.27}{0.61} & \longcontextscore{6.98}{0.30} & \longcontextscore{6.11}{0.05} & 6.79 & 4.72 \\
\bottomrule
\end{tabular}
\end{center}
\end{table}

\begin{table}[!htb]
\caption{Qwen3-8B MRCR accuracy (\%) by needle count and context-length bucket.}
\label{tab:mrcr_by_length}
\begin{center}
\begin{tabular}{@{}l|ccccc@{}}
\toprule
Method & 0k-8k & 8k-16k & 16k-32k & 32k-64k & 64k-128k \\
\midrule
\emph{2-Needle} \\
BF16 & \longcontextscore{54.72}{2.40} & \longcontextscore{41.14}{1.01} & \longcontextscore{37.69}{3.47} & \longcontextscore{27.48}{2.17} & \longcontextscore{23.16}{3.31} \\
WUSH-KV & \textbf{\longcontextscore{18.05}{0.26}} & \textbf{\longcontextscore{10.82}{0.92}} & \textbf{\longcontextscore{10.14}{0.18}} & \textbf{\longcontextscore{8.14}{0.79}} & \textbf{\longcontextscore{5.49}{0.14}} \\
OSCAR & \longcontextscore{15.36}{1.30} & \longcontextscore{8.88}{0.56} & \longcontextscore{7.05}{0.40} & \longcontextscore{4.69}{0.69} & \longcontextscore{0.29}{0.11} \\
\midrule
\emph{4-Needle} \\
BF16 & \longcontextscore{29.64}{1.71} & \longcontextscore{22.73}{1.16} & \longcontextscore{23.59}{1.26} & \longcontextscore{19.48}{3.17} & \longcontextscore{17.14}{2.10} \\
WUSH-KV & \textbf{\longcontextscore{12.83}{0.36}} & \longcontextscore{9.15}{0.08} & \textbf{\longcontextscore{10.93}{1.21}} & \textbf{\longcontextscore{7.26}{0.25}} & \textbf{\longcontextscore{7.87}{0.41}} \\
OSCAR & \longcontextscore{11.63}{1.18} & \textbf{\longcontextscore{9.44}{0.65}} & \longcontextscore{8.66}{0.25} & \longcontextscore{3.61}{0.24} & \longcontextscore{1.06}{0.12} \\
\midrule
\emph{8-Needle} \\
BF16 & \longcontextscore{18.22}{1.39} & \longcontextscore{20.67}{1.39} & \longcontextscore{17.04}{1.38} & \longcontextscore{13.78}{1.31} & \longcontextscore{12.25}{0.56} \\
WUSH-KV & \longcontextscore{8.99}{0.19} & \textbf{\longcontextscore{9.73}{0.74}} & \textbf{\longcontextscore{8.36}{0.46} }& \textbf{\longcontextscore{8.19}{0.24}} & \textbf{\longcontextscore{6.98}{0.15}} \\
OSCAR & \textbf{\longcontextscore{9.34}{0.37}} & \longcontextscore{8.63}{0.38} & \longcontextscore{7.11}{0.78} & \longcontextscore{4.26}{0.79} & \longcontextscore{1.09}{0.27} \\
\bottomrule
\end{tabular}
\end{center}
\end{table}

\begin{table}[!htb]
\caption{Qwen3-32B MRCR accuracy (\%) by needle count, including context-length buckets, overall scores, and AUC.}
\label{tab:mrcr_32b}
\begin{center}
\begin{tabular}{@{}l|ccccccc@{}}
\toprule
Method & 0k-8k & 8k-16k & 16k-32k & 32k-64k & 64k-128k & Overall & AUC \\
\midrule
\emph{2-Needle} \\
BF16 & 70.98 & 60.14 & 50.88 & 38.27 & 36.87 & \longcontextscore{51.49}{2.03} & 43.70 \\
WUSH-KV & 16.39 & 10.79 & \textbf{8.93} & \textbf{10.48} & \textbf{6.18} & \longcontextscore{10.57}{0.69} & \textbf{9.25} \\
OSCAR & \textbf{24.83} & \textbf{14.39} & \textbf{9.01} & \textbf{10.87} & 3.92 & \textbf{\longcontextscore{12.64}{0.79}} & \textbf{9.46 }\\
\midrule
\emph{4-Needle} \\
BF16 & 47.80 & 27.65 & 23.84 & 28.62 & 24.22 & \longcontextscore{30.75}{1.76} & 27.03 \\
WUSH-KV & \textbf{12.25} & 4.79 & 8.26 & \textbf{8.41} & \textbf{8.44} & \longcontextscore{8.50}{0.54} & \textbf{8.15} \\
OSCAR & \textbf{14.84} & \textbf{12.59} & \textbf{9.41} & 7.48 & 5.51 & \textbf{\longcontextscore{10.06}{0.60}} & \textbf{8.10} \\
\midrule
\emph{8-Needle} \\
BF16 & 29.41 & 21.58 & 17.50 & 25.01 & 15.32 & \longcontextscore{21.86}{1.45} & 20.72 \\
WUSH-KV & 7.79 & 7.39 & 7.25 & \textbf{8.43} & \textbf{6.95} & \longcontextscore{7.57}{0.45} & \textbf{7.68 }\\
OSCAR & \textbf{9.87} & \textbf{9.67}& \textbf{8.98} & \textbf{8.29} & 4.62 & \textbf{\longcontextscore{8.29}{0.41}} & \textbf{7.64} \\
\bottomrule
\end{tabular}
\end{center}
\end{table}

\end{document}